%% file: main.tex
\documentclass[10pt,twocolumn,letterpaper]{article}

\usepackage[pagenumbers]{style}     
\definecolor{styleblue}{rgb}{0.21,0.49,0.74}
\usepackage[pagebackref,breaklinks,colorlinks,allcolors=styleblue]{hyperref}

\input{preamble}

\title{\textit{Off the Rails}: Hijacking the Scoring Head in Generative End-to-End Driving Planners with Safety-Violating Adversarial Perturbations}

\author{Halima Bouzidi\textsuperscript{*} \;\; Mboutidem Mkpong \;\; Haoyu Liu \;\; Mohammad Al Faruque \\
University of California, Irvine, CA, USA \\
{\tt\small \{hbouzidi, mkpongm, hliu32, alfaruqu\}@uci.edu}
}

\begin{document}
\maketitle
\renewcommand{\thefootnote}{\fnsymbol{footnote}}
\footnotetext[1]{Corresponding Author.}

\input{sec/0_abstract}    
\input{sec/1_intro}
\input{sec/2_related_work}

\input{sec/3_approach}
\input{sec/4_evaluation}

\input{sec/5_conclusion}

{
    \small
    \bibliographystyle{ieeenat_fullname}
    \bibliography{main}
}

\newpage

\input{sec/X_supp}

\end{document}

%% file: preamble.tex
\usepackage{graphicx}
\usepackage{booktabs}

\usepackage{multirow}      
\usepackage[table]{xcolor} 
\usepackage{colortbl}      
\usepackage{caption}       
\usepackage{titletoc} 
\usepackage{cuted} 

\usepackage{amsmath, amsthm, amssymb}

\newtheorem{theorem}{Theorem}
\newtheorem{corollary}{Corollary}
\newtheorem{proposition}{Proposition}
\newtheorem{definition}{Definition}  
\newtheorem{remark}{Remark}          

\usepackage[accsupp]{axessibility}

\usepackage{orcidlink}

\definecolor{HeaderGray}{RGB}{235,235,235}
\definecolor{CleanGreen}{RGB}{213,233,213}
\definecolor{DeepRed}{RGB}{160,30,30}
\definecolor{DarkOrange}{RGB}{210,110,0}

\definecolor{better}{RGB}{34,139,34}   
\definecolor{worse}{RGB}{160,30,30}   

%% file: sec/0_abstract.tex
\begin{abstract}

Generative models have recently seen rapid adoption in End-to-End (E2E) autonomous driving (AD), with diffusion-based denoising and vocabulary-based retrieval becoming the dominant trajectory-decoding paradigms. Despite their architectural diversity, current generative AD planners share a common inference pattern: a fixed set of candidate trajectories (anchors, vocabulary entries, or proposal queries) is scored by one or more learned heads conditioned on the Bird's-Eye-View (BEV) features, and the highest-scored candidate is returned as the final trajectory. Under this design, the scoring head is the only barrier between perception and the motion command, and its decision margins between competing candidates are often small. We introduce \textsc{Derail}, an adversarial framework that exploits this scoring-head attack surface. Evaluated on various generative planners, \textsc{Derail} flips the trajectory selection from a safe to an unsafe candidate, with score drops of $39$--$80\%$ and collision rates of up to $50\%$, consistently outperforming generic loss-maximization and feature-divergence attacks. Our analysis suggests that safety-violating objectives govern attack effectiveness against generative AD planners, and that the scoring-head inference pattern itself is a recurring attack surface worth explicit defensive consideration.

\end{abstract}

%% file: sec/1_intro.tex
\section{Introduction}

Conditional generative modeling \cite{sohn2015learning, li2019conditional, ajayconditional, songscore, he2023diffusion} has reshaped policy learning, offering a principled framework to capture the inherent multi-modality of complex autonomous applications \cite{ivanovic2020multimodal, gomez2020real, wangrobogen, zhang2025generative, wang2025generative}. In Autonomous Driving (AD), generative planners, spanning diffusion-based denoising and vocabulary-based retrieval, have become dominant trajectory-decoding architectures due to their capacity to model non-Gaussian action distributions and outperform regression baselines \cite{ho2020denoising, janner2022planning, mishra2023generative, li2024crossway, luo2024pot, chi2025diffusion, arnelid2019recurrent, zheng2024genad, wang2025generative}. Conditioned on multimodal sensor data such as multi-camera imagery and derived Bird's-Eye-View (BEV) features, these systems generate trajectories that account for both road topology and the non-linear interactions of surrounding agents \cite{zhengdiffusion, liao2025diffusiondrive, wang2025diffad, li2025generalized, zheng2026unleashing, yao2026reflexdiffusion, reddy2026rapid, liu2024ddm}.

\noindent With the increasing adoption of generative planners in commercial AD systems~\cite{jiang2023motiondiffuser}, their adversarial robustness has become a first-order safety question. Because generative policies are strictly context-dependent, they inherit the well-known vulnerabilities of deep visual encoders \cite{szegedy2013intriguing, pattanaik2018robust, gleave2019adversarial, sun2020stealthy, mo2022attacking, bouzidi2025see, bouzidi2026out}. Existing adversarial work on generative policies, however, has been developed primarily for robotic manipulation, where attacks such as DP-Attacker~\cite{chen2024diffusion} target the noise-prediction loss inside long denoising chains. This methodology is misaligned with AD planners, which run under tight real-time budgets, use very short denoising or single-step decoding schedules, and, more importantly, do not actually output a freely refined continuous sample: they ultimately \emph{select} a discrete trajectory from a fixed candidate set via learned scoring. The relevant attack surface in AD is therefore not the denoising chain itself, but the scoring rule that turns visual features into a trajectory selection.

\noindent The architectures of prominent recent generative AD planners~\cite{liao2025diffusiondrive, li2025generalized} reveal a consistent inference pattern that we abstract into five stages, regardless of whether the decoder is described as \textit{diffusion} \cite{liao2025diffusiondrive} or \textit{vocabulary} based \cite{li2025generalized}: (1) the multi-view camera image is encoded into a BEV feature map; (2) a \emph{fixed} set of candidate trajectories is instantiated, either as a small anchor set or as a dense trajectory vocabulary; (3) one or more learned heads compute per-candidate scores conditioned on the BEV features; (4) the scoring head selects the highest-scored candidate; (5) no post-hoc geometric or kinematic safety filter is applied. Stages (1) and (3) are fully visual-feature-dependent and differentiable, and the scoring head's final selection (stage 4) is preceded by score margins that, in practice, are small enough that input-space perturbations within the limited budget can flip the selection from a safe candidate to an unsafe one. The number of denoising steps, the size of the candidate set, and the number of scoring heads change the decision \emph{margin} but not the underlying surface: perception flows into a learned score, and the score directly decides the final trajectory in every planner \cite{liao2025diffusiondrive, li2025generalized}.

\noindent \textbf{Scientific Contributions.} We present a vulnerability study of the scoring-head attack surface in generative E2E AD planners. (i) We identify the scoring-head pattern, a learned scoring head over a fixed candidate set with no post-hoc safety filter, as a recurring attack surface that is shared across diffusion- and vocabulary-based AD planners despite their reported architectural differences. (ii) We introduce \textsc{Derail}, a ground-truth-free, safety-violating objective composed of three behavior-level terms that are explicitly aligned with the safety axes of the NAVSIM \cite{dauner2024navsim} safety benchmark, and show that this safety-violating objective is significantly more effective than generic loss-maximization, feature-divergence, and native-objective-maximization~\cite{chen2024diffusion} baselines under the same threat model. (iii) We provide a unified, fully differentiable adversarial attack pipeline that supports both per-frame digital perturbations and simulated, object-aware adversarial patches, allowing the same objective to be deployed across all evaluated planners without per-model tuning. (iv) We evaluate \textsc{Derail} on four state-of-the-art generative AD planners~\cite{liao2025diffusiondrive, li2025generalized} spanning diffusion-based denoising and vocabulary-based retrieval, and show consistent collision rate inflation (up to $50\%$).

%% file: sec/2_related_work.tex
\section{Related Work}

\textbf{Generative Models for Policy Generation.}
Generative modeling has emerged as a principled framework for policy learning in complex autonomous systems \cite{sohn2015learning, ajayconditional, songscore}. While classical imitation learning \cite{schaal1999imitation, pan2018agile, osa2018algorithmic} relies on deterministic regression that suffers from mode averaging \cite{ivanovic2020multimodal, gomez2020real, wangrobogen}, generative models (e.g., diffusion, flow matching) have become standard for behavior cloning and trajectory generation \cite{ho2020denoising, janner2022planning, mishra2023generative, li2024crossway, luo2024pot, chi2025diffusion}. Modern AD planning models \cite{zhengdiffusion, liao2025diffusiondrive, wang2025diffad, li2025generalized, zheng2026unleashing, yao2026reflexdiffusion} incorporate high-level semantic priors such as motion anchors, trajectory vocabularies, and BEV spatial features to meet real-time constraints. We focus on the robustness of generative AD planners, a technology already integrated into commercial-grade systems such as \textit{Waymo}'s MotionDiffuser \cite{jiang2023motiondiffuser} and \textit{Nvidia}'s Cosmos \cite{ren2025cosmos}.

\noindent \textbf{Adversarial Robustness of Generative Planning.}
The vulnerability of deep neural networks to adversarial attacks is well documented \cite{szegedy2013intriguing, gleave2019adversarial, mo2022attacking, pattanaik2018robust, sun2020stealthy}. Attacking generative policies, however, has been studied almost exclusively in robotic manipulation: the recent DP-Attacker \cite{chen2024diffusion} targets camera inputs by maximizing the noise-prediction loss of long denoising chains. Such designs are misaligned with AD planning models, which run very short denoising or single-step decoding schedules and ultimately select a discrete trajectory from a fixed candidate set via learned scoring rather than refining a continuous noise sample. As a result, attacks tuned to maximize per-step noise prediction error often fail to flip the discrete selection that determines the actual planner final motion command.

\noindent \textbf{Vulnerabilities in E2E Autonomous Driving.}
Adversarial research on E2E AD has historically targeted regression-based pipelines such as \textit{UniAD} \cite{hu2023planning} or vectorized planners such as \textit{VAD} \cite{jiang2023vad}. These studies expose weaknesses in perception encoders \cite{wang2024attack, wu2023adversarial, zhang2024uniada, chahe2023dynamic} but stop at perceptual degradation (e.g., detection drops) without examining how downstream trajectory selection amplifies such corruptions. Generative planners differ along this dimension: their inference reduces to scoring a fixed candidate set and selecting via the scoring head, so small perceptual perturbations can trigger discrete mode flips that turn a minor feature corruption into a full trajectory reassignment toward an unsafe candidate. In the physical domain, prior work has demonstrated adversarial patches against object detectors \cite{brown2017adversarial, liu2018dpatch}, but adversarial attacks specifically targeting the scoring-head attack surface of generative AD planners have, to our knowledge, not been studied before. \textsc{Derail} provides a first study of how this scoring-head inference pattern, shared across diffusion-based and vocabulary-based AD planners, can be exploited by safety-violating objectives to induce unsafe trajectories that perception-only attacks cannot reach.

%% file: sec/3_approach.tex
\begin{figure*}[t]
\centering
\includegraphics[width=1.\textwidth]{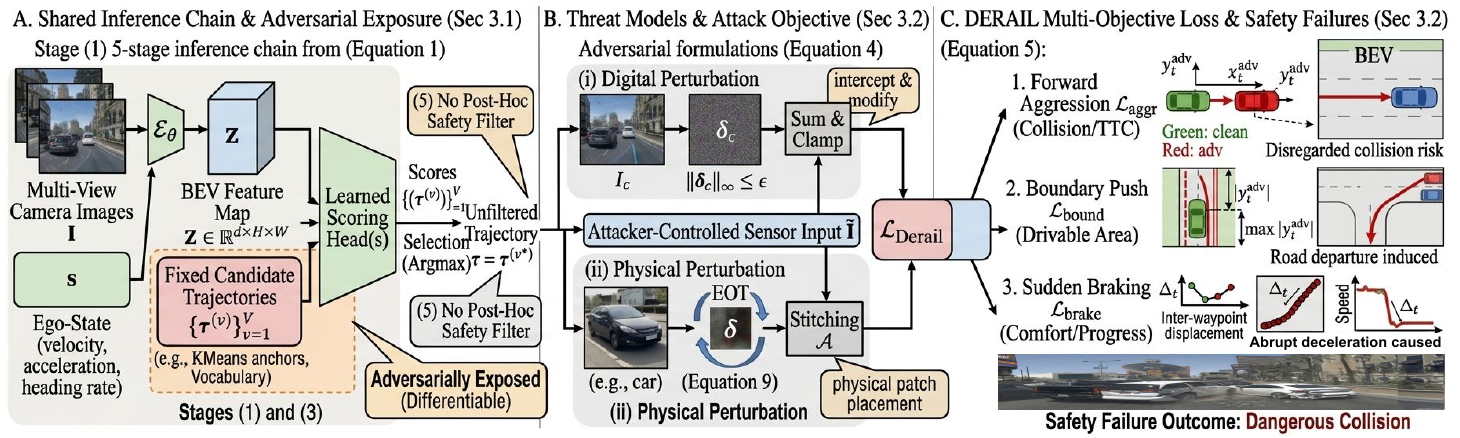}
\vspace{-0.5cm}
\caption{Overview of the \textsc{Derail} framework, targeting generative E2E AD systems.}
\vspace{-0.4cm}
\label{fig:framework}
\end{figure*}

\section{\textsc{Derail}: Framework Design}

\subsection{Preliminaries: Generative Trajectory Planning}
\label{sec:preliminaries}
We consider end-to-end generative AD systems that map raw sensor observations to planned trajectories via learned neural networks. Let $\mathbf{I} = \{I_c\}_{c=1}^{C}$ denote the set of multi-view camera images from $C$ surrounding cameras, and let $\mathbf{s} \in \mathbb{R}^{d_s}$ denote the ego-vehicle state (velocity, acceleration, heading rate). A generative trajectory planning model $f_\theta$ parameterized by weights $\theta$ produces a planned trajectory $\boldsymbol{\tau} = [\mathbf{p}_1, \mathbf{p}_2, \ldots, \mathbf{p}_T] \in \mathbb{R}^{T \times 3}$, where each waypoint $\mathbf{p}_t = (x_t, y_t, \psi_t)$ specifies the planned position and heading at future timestep $t$ over a horizon of $T$ steps.

\vspace{4pt}
\noindent \textbf{The Shared Inference Chain.}
Recent generative planners like DiffusionDrive~\cite{liao2025diffusiondrive} and GTRS~\cite{li2025generalized}, although designed under different headings (diffusion-based denoising for DiffusionDrive, vocabulary-based retrieval for the GTRS variants), instantiate a single five-stage inference chain:
\begin{equation}
    \mathbf{I} \xrightarrow{\mathcal{E}_\theta} \mathbf{Z} \xrightarrow{\mathrm{decoder}} \{(\boldsymbol{\tau}^{(v)}, \ell_v)\}_{v=1}^{V} \xrightarrow{\text{scoring head}} \boldsymbol{\tau}.
    \label{eq:chain}
\end{equation}
\noindent (1) A \emph{visual encoder} $\mathcal{E}_\theta$ maps multi-view images $\mathbf{I}$ to a BEV feature map $\mathbf{Z}\in\mathbb{R}^{d \times H \times W}$. (2) A \emph{fixed} candidate set $\{\boldsymbol{\tau}^{(v)}\}_{v=1}^{V}$ is materialized at inference time: KMeans anchors for DiffusionDrive, and a vocabulary trajectories for the GTRS variants. The candidate set is independent of $\mathbf{Z}$. (3) One or more \emph{learned scoring heads} compute a per-candidate score $\ell_v$ conditioned on $\mathbf{Z}$. In DiffusionDrive this is a single classification head on top of DDIM-refined anchor features; in GTRS-DP it is a single imitation head; in GTRS-Aug and GTRS-Dense it is a multi-head aggregate combining imitation with predictions of collision, drivable-area compliance, time-to-collision, ego progress, and lane keeping, each implemented as a head on $\mathbf{Z}$ rather than a closed-form geometric check. (4) The scoring head selects the highest-scored candidate $v^\star$, $\boldsymbol{\tau} = \boldsymbol{\tau}^{(v^\star)}$. (5) No post-hoc geometric or kinematic safety filter at the planner level is applied after generating the final trajectory.

\vspace{4pt}
\noindent \textbf{Why This Chain is Adversarially Exposed?}
Stages~(1) and~(3) are end-to-end differentiable in $\mathbf{I}$, and stage~(5) is empty. The scoring head's final selection (stage~4) is non-differentiable, and standard practice handles this with a softmax relaxation $\hat{\boldsymbol{\tau}}_\beta = \sum_v \mathrm{softmax}(\beta \ell_v)\,\boldsymbol{\tau}^{(v)}$ so that adversarial gradients can be propagated end-to-end without changing the inference path used at evaluation time. As a result, perturbations to $\mathbf{I}$ propagate into shifts of the entire score vector $\boldsymbol{\ell}$, and a switch of the scoring head's selection from a safe candidate $v^\star$ to an unsafe candidate $v_{\mathrm{adv}}$ only requires shifting the relative score $\ell_{v^\star} - \ell_{v_{\mathrm{adv}}}$ across zero. In planners like \cite{liao2025diffusiondrive, li2025generalized}, this margin is small relative to the perturbation that can induce in $\mathbf{Z}$. We make this argument explicit and bound it in the Appendix; here we use it only to motivate the attack objective. We emphasize that this is a \emph{design-enabled} attack surface rather than a property of any individual decoder family: any planner that selects from a fixed candidate set via a single neural scoring head and applies no post-hoc safety filter inherits the same surface.

\subsection{Adversarial Attack Formulation}
\label{sec:attack}

We now formalize the \textsc{Derail} adversarial attack on generative trajectory planning models. We consider two complementary threat models, \emph{digital} and \emph{physical} that together expose the vulnerability surface of these systems.

\vspace{4pt}
\noindent \textbf{Threat Model.}
We assume a white-box attacker with full access to the planning model weights $\theta$ and architecture; black-box transferability is reported in the Appendix. The white-box assumption is chosen deliberately to characterize the attack surface that the architectural chain in Eq.~\eqref{eq:chain} \emph{permits} under the strongest gradient access a defender should plan against. The attacker crafts adversarial perturbations to the camera inputs with the goal of inducing dangerous driving behavior, and we formalize two attack modalities:

\vspace{4pt}
\noindent (\emph{i}) \textit{Digital perturbation.} Attacker applies an additive, pixel-wise perturbation $\boldsymbol{\delta}_c \in \mathbb{R}^{3 \times H_I \times W_I}$ to a target camera view, bounded by an $\ell_\infty$-norm constraint:
\begin{equation}
    \tilde{I}_c = \text{clamp}\!\left(I_c + \boldsymbol{\delta}_c,\; 0,\; 1\right), \quad \|\boldsymbol{\delta}_c\|_\infty \leq \epsilon,
    \label{eq:digital_attack}
\end{equation}
where $\epsilon$ controls the imperceptibility budget. This models a cyber-attacker who can intercept and modify the camera feed before it reaches the AD system.

\vspace{4pt}
\noindent (\emph{ii}) \textit{Physical perturbation.} Attacker places an adversarial patch $\boldsymbol{\delta} \in [0,1]^{3 \times h_p \times w_p}$ on a physical surface visible to the ego-vehicle. A stitching operator $\mathcal{A}$ composites the patch onto the image at the appropriate location and scale:
\begin{equation}
    \tilde{I}_c = \mathcal{A}(\boldsymbol{\delta},\, I_c,\, \mathbf{b}_c),
    \label{eq:physical_attack}
\end{equation}
where $\mathbf{b}_c$ encodes the placement geometry (position, scale) on camera $c$, modeling a real-world adversary who attaches a patch to a nearby object (e.g., vehicle). 

\vspace{4pt}
\noindent \textbf{\textsc{Derail} Attack Objective.}
Let $\boldsymbol{\tau}_{\text{clean}} = f_\theta(\mathbf{I}, \mathbf{s})$ denote the trajectory produced from clean inputs, and $\boldsymbol{\tau}_{\text{adv}} = f_\theta(\tilde{\mathbf{I}}, \mathbf{s})$ the trajectory under the adversarial perturbation. The attacker's goal under Eq.~\eqref{eq:chain} is to make the scoring head select a candidate that is unsafe under the NAVSIM PDM Score. Rather than expressing this as a discrete combinatorial objective over the candidate set, we use three differentiable, behavior-level surrogates that act on the softmax-relaxed trajectory $\hat{\boldsymbol{\tau}}_\beta$ and that are explicitly aligned with an unsafe behavior (collision, off-road, driver discomfort). Regardless of the attack modality, the attacker solves:
\begin{equation}
    \max_{\boldsymbol{\delta} \in \mathcal{P}} \; \mathcal{L}_{\text{Derail}}(\boldsymbol{\delta}),
    \label{eq:attack_obj}
\end{equation}
where $\mathcal{P}$ denotes the feasible perturbation set: $\mathcal{P} = \{\boldsymbol{\delta} : \|\boldsymbol{\delta}\|_\infty \leq \epsilon\}$ for digital attacks, or $\mathcal{P} = [0,1]^{3 \times h_p \times w_p}$ for physical patches. The \textsc{Derail} loss is a weighted combination of three safety-violating components, each chosen to push the relaxed trajectory toward a different failure axis:
\begin{equation}
    \mathcal{L}_{\text{Derail}}(\boldsymbol{\delta}) = w_{\text{aggr}} \cdot \mathcal{L}_{\text{aggr}} + w_{\text{bound}} \cdot \mathcal{L}_{\text{bound}} +  w_{\text{brake}} \cdot \mathcal{L}_{\text{brake}}
    \label{eq:Derail_loss}
\end{equation}

\vspace{3pt}
\noindent \textit{(i) Forward Aggression Loss} $\mathcal{L}_{\text{aggr}}$ (targets NAVSIM Collision and TTC sub-metrics). This component pushes the planning model toward aggressive forward motion that ignores surrounding agents. It maximizes longitudinal displacement while suppressing lateral corrections:
\begin{equation}
    \mathcal{L}_{\text{aggr}} = \frac{1}{T}\sum_{t=1}^{T} x_t^{\text{adv}} - \gamma \sum_{t=1}^{T} |y_t^{\text{adv}}|,
    \label{eq:aggression_loss}
\end{equation}
where $x_t^{\text{adv}}$ and $y_t^{\text{adv}}$ are the longitudinal and lateral coordinates, and $\gamma$ controls the lateral suppression, inducing reckless driving that disregards collision risks.

\vspace{3pt}
\noindent \textit{(ii) Boundary Push Loss} $\mathcal{L}_{\text{bound}}$ (targets NAVSIM Drivable-Area and Driving-Direction sub-metrics). This component steers the trajectory toward the edges of the drivable area by maximizing lateral deviation:
\begin{equation}
    \mathcal{L}_{\text{bound}} = \frac{1}{T}\sum_{t=1}^{T} |y_t^{\text{adv}}| + \mu \cdot \max_t\, |y_t^{\text{adv}}|,
    \label{eq:boundary_loss}
\end{equation}
where the first term encourages sustained lateral drift and the second amplifies the peak excursion, inducing road departure or wrong-lane driving.

\vspace{3pt}
\noindent \textit{(iii) Sudden Braking Loss} $\mathcal{L}_{\text{brake}}$ (targets NAVSIM Comfort and Ego-Progress sub-metrics). Induces abrupt deceleration by penalizing forward velocity in the adversarial trajectory. Given the planned waypoints $\{\mathbf{p}_t^{\text{adv}}\}_{t=1}^{T}$, we compute the inter-waypoint displacements $\Delta_t = \|\mathbf{p}_{t+1}^{\text{adv}} - \mathbf{p}_t^{\text{adv}}\|_2$ and minimize forward progress while amplifying speed discontinuities:
\begin{equation}
    \mathcal{L}_{\text{brake}} = -\frac{1}{T\!-\!1}\sum_{t=1}^{T-1} \Delta_t + \lambda_{\text{jerk}} \sum_{t=2}^{T-1} \left(\Delta_t - \Delta_{t-1}\right)^2,
    \label{eq:brake_loss}
\end{equation}

\noindent \textbf{Expectation over Transformations.}
For physical patch attacks, robustness to variation in viewpoint and placement geometry is critical. We optimize under an Expectation over Transformations (EOT)~\cite{athalye2018synthesizing} framework, applying stochastic geometric augmentations $\mathbf{t} \sim \mathcal{T}$ (rotation, translation, scaling, shearing) 
\begin{equation}
    \mathcal{L}_{\text{Derail}}^{\text{EOT}}(\boldsymbol{\delta}) = \mathbb{E}_{\mathbf{t} \sim \mathcal{T}}\!\left[\, \mathcal{L}_{\text{Derail}}\!\left(\mathbf{t}(\boldsymbol{\delta})\right)\right],
    \label{eq:eot}
\end{equation}
approximated via a single stochastic sample per step. For digital attacks, the perturbation is applied directly without EOT augmentation.

\noindent \textbf{Optimization.}
The optimization is adapted to each attack modality. For \emph{digital attacks}, we employ Projected Gradient Descent (PGD)~\cite{madry2018towards} with step size $\alpha$:
\begin{equation*}
    \boldsymbol{\delta}^{(t+1)} = \Pi_{\|\cdot\|_\infty \leq \epsilon}\!\left[\boldsymbol{\delta}^{(t)} + \alpha \cdot \text{sign}\!\left(\nabla_{\boldsymbol{\delta}}\, \mathcal{L}_{\text{Derail}}\right)\right],
    \label{eq:pgd}
\end{equation*}
where $\Pi$ projects the perturbation back onto the $\ell_\infty$-ball.

\noindent For \emph{physical patch attacks}, we use Adam~\cite{kingma2014adam} with gradient accumulation across all $N$ scene tokens in a driving scenario.
\begin{equation*}
    \boldsymbol{\delta}^{(e+1)} = \Pi_{[0,1]} \left[ \boldsymbol{\delta}^{(e)} + \alpha \cdot \textsc{Adam} \left( \frac{1}{N}\sum_{n=1}^{N} \nabla_{\boldsymbol{\delta}}\, \mathcal{L}_{\text{Derail}}^{\text{EOT}} \right) \right],
    \label{eq:adam}
\end{equation*}
$\Pi_{[0,1]}$ projects patch pixels to valid ranges. The patch is initialized from $\boldsymbol{\delta}^{(0)} \sim \mathcal{U}(0,1)$ and optimized over $E$ epochs.

%% file: sec/4_evaluation.tex
\section{Evaluation}

\subsection{Experimental Setup}

\noindent \textbf{Generative Planning Models and Datasets.}
We evaluate the \textsc{Derail} attack against four state-of-the-art generative planning models on the NAVSIM benchmark~\cite{dauner2024navsim}. The target models are: (\emph{i}) DiffusionDrive~\cite{liao2025diffusiondrive}, a truncated denoising diffusion planner that scores KMeans anchor trajectories with a single classification head after $K{=}2$ DDIM refinement steps. (\emph{ii}) GTRS-DP~\cite{li2025generalized}, a vocabulary-based retrieval planner that scores a $4{,}096$-trajectory vocabulary with a single imitation head. (\emph{iii}) GTRS-Aug~\cite{li2025generalized}, an augmented variant of GTRS that combines the imitation head with multiple safety-related scoring heads (collision, drivable-area, time-to-collision, ego progress) over an $8{,}192$-trajectory vocabulary. (\emph{iv}) GTRS-Dense~\cite{li2025generalized}, a dense prediction variant that applies the same multi-head safety scorer over $16{,}384$ vocabulary candidates without pruning. All four planners follow the scoring-head inference chain of Eq.~\eqref{eq:chain}: a fixed candidate set is scored by one or more learned heads conditioned on the BEV features, and the final trajectory is selected by the scoring head with no post-hoc geometric or kinematic safety filter. Each model operates on a multi-view camera setup with left, front, and right cameras, which are cropped and stitched into a panoramic image before being fed to the visual encoder.

\vspace{4pt}
\noindent \textbf{Metrics.}
We adopt the Planning Decision-Making Score (PDMS)~\cite{dauner2023parting} as our evaluation metric. PDMS provides a comprehensive assessment of planning quality through a combination of safety and comfort sub-scores:  (\emph{i}) \textit{Collision Rate (CR)}: the fraction of scenarios with at-fault collisions.  (\emph{ii}) \textit{Off-Road Rate (OR)}: the fraction of scenarios violating drivable area compliance. (\emph{iii}) \textit{Wrong-Way Rate (WW)}: the fraction of scenarios violating driving direction compliance. (\emph{iv}) \textit{Unsafe Time-To-Collision rate (TTC)}: the fraction of scenarios where the ego vehicle's time-to-collision falls below a safe threshold. (\emph{v}) \textit{Comfort (CM)}: fraction of scenarios with jerk and acceleration within acceptable bounds. 

\vspace{4pt}
\noindent \textbf{Baselines.}
We compare against three adversarial baselines under the same threat model on a single camera view: \textit{Untargeted PGD}~\cite{madry2018towards} applies standard Projected Gradient Descent to maximize the model's evaluation-mode training loss via gradient ascent on camera inputs. \textit{Encoder Attack} maximizes L2 and cosine distance between clean and adversarial vision encoder features. \textit{DP-Attacke}r~\cite{chen2024diffusion} maximizes the training-mode denoising loss at randomly sampled generative timestamps. 
Unlike these baselines, which rely on the model's own training objective or intermediate features, \textsc{Derail} uses a ground-truth-free, multi-objective loss designed to induce specific unsafe driving behaviors

\vspace{4pt}
\noindent \textbf{Digital Perturbation Setting.}
For digital attacks, we apply PGD~\cite{madry2018towards} directly on the raw camera pixel values. The perturbation is constrained to an $\ell_\infty$-ball with radius $\epsilon = 8/255$ and optimized with step size $\alpha = 2/255$ over $K{=}20$ iterations. The attack targets a single camera view (front camera by default) while leaving the remaining views unperturbed (See Appendix for multi-camera analysis). The perturbation is updated via signed gradient ascent on $\mathcal{L}_{\text{Derail}}$.

\vspace{4pt}
\noindent \textbf{Physical Perturbation Setting.}
For physical attacks, we optimize a universal adversarial patch $\boldsymbol{\delta} \in [0,1]^{3 \times 256 \times 256}$ that is composited onto the camera image via differentiable stitching. The patch is placed using \emph{object-aware placement}: the 3D bounding boxes of surrounding vehicles are projected onto the camera plane. The patch is resized to cover $\rho{=}0.5$ of the target vehicle's bounding box. For physical robustness, the patch undergoes stochastic geometric augmentations at each training step (Details in Appendix).

\vspace{4pt}
\noindent \textbf{Optimization Details.}
We optimize a single universal patch across all scenes using Adam~\cite{kingma2014adam}. The same \textsc{Derail} loss weights ($w_{\text{brake}}$, $w_{\text{aggr}}$, $w_{\text{bound}}$) are used across all models and both attack modalities. All experiments are conducted on a single NVIDIA RTX 3060 GPU. Crafting digital perturbation takes 1-2 sec. Training a single patch takes roughly 10-15 minutes depending on the target model.

\subsection{Experimental Results}

\subsection{Digital Attack Evaluation}
\noindent Table~\ref{tab:digital_results} summarizes digital attack results across the four planning models. \textsc{Derail} is the only method achieving $100\%$ Attack Success Rate (ASR) on every model, with score drops of $39.8$--$80.1\%$ and collision rates of $25.8$--$50.7\%$, consistently outperforming all baselines. The most informative comparison is with DP-Attacker~\cite{chen2024diffusion}, which uses the same PGD optimization loop and the same $\epsilon{=}8/255$ perturbation budget yet only reaches $85$--$100\%$ ASR with significantly smaller score drops ($16$--$34\%$). The two attacks therefore differ only in their objective: DP-Attacker maximizes the model's own training-mode loss (e.g., diffusion noise prediction), which is well aligned with degrading per-step denoising quality but has no relationship to whether the scoring head in Eq.~\eqref{eq:chain} switches to an unsafe candidate. \textsc{Derail}, by contrast, operates directly in the trajectory-output space through behavior-level surrogates (forward aggression, boundary push, sudden braking) that are aligned with the safety axes of the PDM Score, so its gradients consistently push the relaxed trajectory in directions that correlate with flipping the scoring head's selection from a safe to an unsafe candidate. The Encoder Attack is the most direct comparison for a model-agnostic, encoder-targeting baseline: it underperforms \textsc{Derail} on every model and is nearly ineffective on DiffusionDrive ($\text{Drop}{=}-0.12\%$), where the very short ($K{=}2$) DDIM chain still gives the scoring head enough margin to recover from generic feature corruption. GTRS-Dense exhibits the highest resilience among baselines, consistent with its multi-head safety scorer aggregating evidence across several auxiliary predictions; even so, \textsc{Derail} induces a $25.8\%$ collision rate and the highest $43.3\%$ discomfort rate observed in our experiments.

\input{tables/digital_results}

\begin{figure*}[t]
\centering
\includegraphics[width=1.\textwidth]{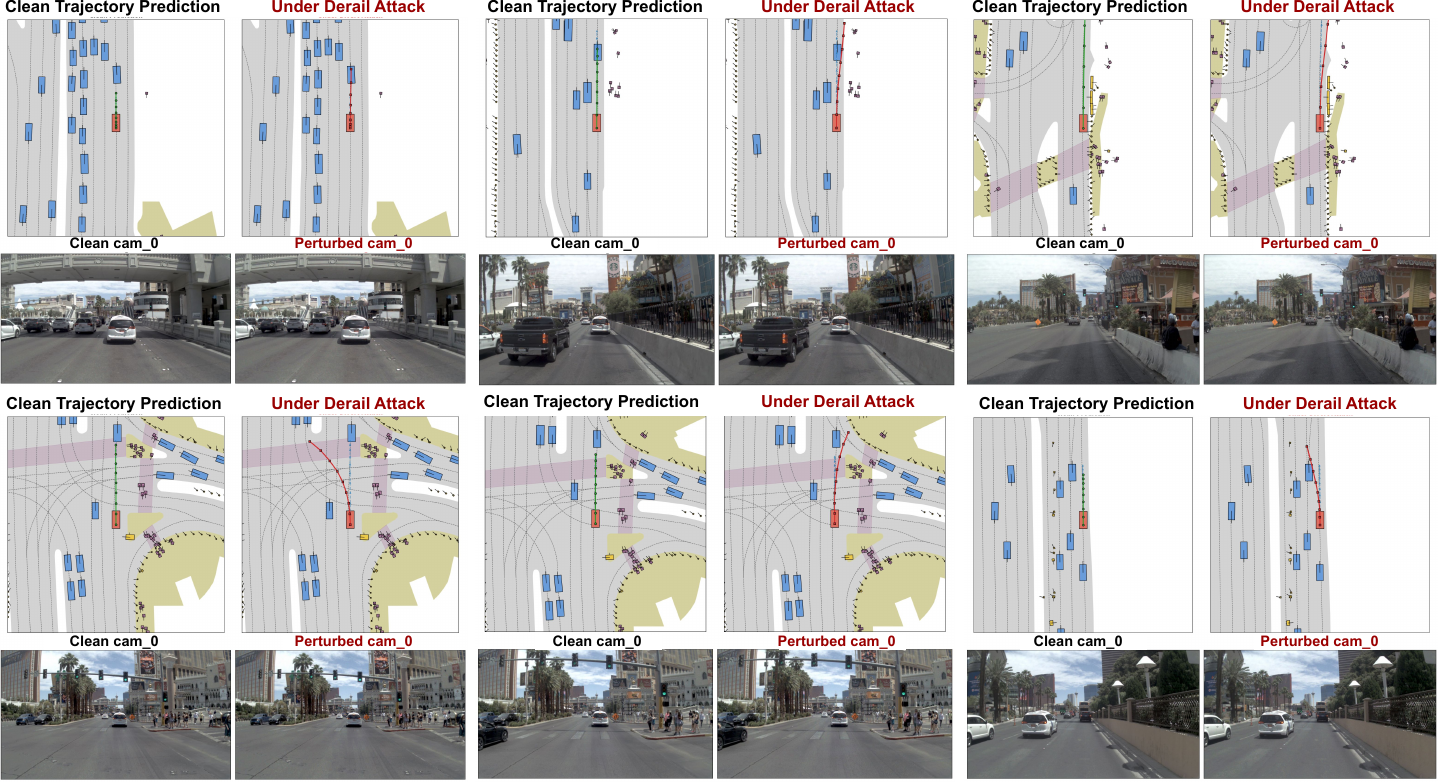}
\vspace{-0.6cm}
\caption{Qualitative visualization of \textsc{Derail} digital attacks across six NAVSIM scenarios. Each scenario shows the BEV trajectory prediction and front-camera view under clean (left) and adversarially digital perturbed (right) conditions. Green trajectories indicate clean predictions; red  trajectories indicate collision-inducing predictions under \textsc{Derail}.}
\label{fig:qualitative}
\vspace{-0.5cm}
\end{figure*}

\subsection{Qualitative Analysis}
Figures~\ref{fig:qualitative} presents scenarios comparing clean and adversarially perturbed trajectory predictions alongside their corresponding front-camera views. Under clean conditions, the planning model (DiffusionDrive~\cite{dauner2024navsim}) consistently produces safe, lane-centered trajectories (green) that maintain appropriate following distance and respect road boundaries. Under \textsc{Derail}, the predicted trajectories (red) exhibit three qualitatively distinct failure modes, each corresponding to the scoring head in Eq.~\eqref{eq:chain} switching to a different family of unsafe candidates in the model's fixed trajectory set. First, in highway and urban following scenarios (top row, columns 1--2), $\mathcal{L}_{\text{aggr}}$ drives the scoring head toward high-speed forward anchors/vocabulary entries that ignore the leading agent, steering the ego vehicle directly into a collision. Second, in scenarios involving curved roads and intersections (bottom row), $\mathcal{L}_{\text{bound}}$ shifts the scoring head toward laterally displaced candidates, producing off-road or wrong-lane trajectories. Third, across multiple scenarios, $\mathcal{L}_{\text{brake}}$ selects candidates with abrupt velocity profiles relative to the smooth clean baseline. Critically, the front-camera views reveal that the adversarial perturbations are visually subtle, the perturbed images appear nearly identical to their clean counterparts, confirming that \textsc{Derail} achieves safety-critical trajectory switching through imperceptible pixel-space modifications without altering the semantic content of the scene.

\input{tables/one_log_patch}

\input{tables/all_log_patch}

\begin{figure*}[h]
\centering
\includegraphics[width=1.\textwidth]{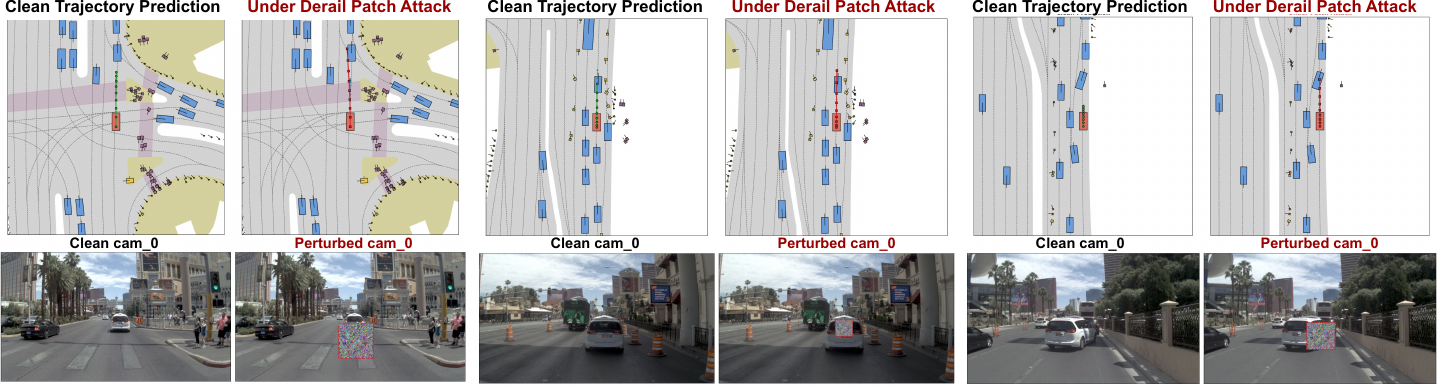}
\vspace{-0.5cm}
\caption{Qualitative visualization of \textsc{Derail} physical patch-based attacks across three NAVSIM scenarios. Each scenario shows the BEV trajectory prediction and front-camera view under clean (left) and adversarially perturbed via physical patch (right) conditions. }
\label{fig:qualitative_patch}
\vspace{-0.3cm}
\end{figure*}

\subsection{Universal Physical Patch Evaluation}

\noindent \textbf{Single Driving Log.} Table~\ref{tab:one_log_patch_results} presents physical patch results on a representative driving log. The collision-focused patch achieves $100\%$ ASR across all four models, while the random-texture patch produces no systematic degradation, despite a spurious $100\%$ ASR on GTRS-AUG, collision rates remain unchanged across all models, confirming that adversarial effectiveness requires targeted optimization and cannot arise from arbitrary visual perturbations. DiffusionDrive and GTRS-Dense suffer the largest collision rate increases ($16.7\%\to33.3\%$ each), with DiffusionDrive additionally exhibiting the largest score drop ($18.5\%$), consistent with its sensitivity to input-space perturbations observed in the digital setting. GTRS-DP also shows meaningful degradation, with its collision rate rising from $8.3\%$ to $25.0\%$ ($\text{Drop}{:}\;15.2\%$). GTRS-AUG proves the most resilient ($\text{Drop}{:}\;10.2\%$), suggesting that augmentation-hardened encoders partially mitigate patch-induced perturbations in the physical setting, though this resilience does not generalize to digital attacks, where the model remains highly vulnerable (see Table~\ref{tab:digital_results}). Notably, off-road and wrong-way rates remain at $0\%$ across all conditions, revealing that physical patches primarily induce frontal collisions without broader trajectory failures, a narrower but safety-critical failure mode that demonstrates the practical danger of physically deployable adversarial patches against generative end-to-end AD systems.

\vspace{4pt}
\noindent \textbf{All Driving Logs.} Table~\ref{tab:all_logs_patch_results} reports results averaged across all $50\%$ driving logs from NAVSIM. DiffusionDrive remains the most vulnerable ($\text{Drop}{:}\;13.4\%$, $\text{CR}{:}\;1.3\%\to13.6\%$), with GTRS-AUG also exhibiting meaningful degradation ($\text{CR}{:}\;2.6\%\to9.1\%$). While GTRS-Dense shows lower sensitivity to patch-based perturbations ($\text{Drop}\leq3.3\%$, $\text{CR}\leq3.7\%$), this is consistent with the constrained nature of universal patch optimization rather than general adversarial robustness, as the model suffers substantial degradation under digital attacks in Table~\ref{tab:digital_results}. The reduced effectiveness relative to pixel-space attacks (collision rates of $13.6\%$ versus $38.4\%$ under \textsc{Derail}) is expected given that a single frame-independent patch must generalize across diverse visual contexts without per-frame adaptation, a significantly harder optimization problem. Importantly, even under this constrained threat model, DiffusionDrive's unsafe TTC rate rises from $4.9\%$ to $22.8\%$, confirming that universal physical patches constitute a realistic and meaningful safety threat.

\subsection{Ablation Study}

\vspace{4pt}
\noindent \textbf{Adversarial Loss Components.}
\noindent Table~\ref{tab:ablation_loss} reports the contribution of each loss component by ablating $\mathcal{L}_{\text{aggr}}$, $\mathcal{L}_{\text{bound}}$, and $\mathcal{L}_{\text{brake}}$ individually across the four planning models. Removing $\mathcal{L}_{\text{aggr}}$ produces the largest drop in collision rate across all models (e.g., DiffusionDrive: $63.29\%\to40.51\%$, GTRS-Dense: $48.73\%\to17.72\%$), confirming that the longitudinal axis is the dominant driver of frontal collisions. Removing $\mathcal{L}_{\text{bound}}$ substantially reduces off-road rate across all models (e.g., GTRS-DP: $82.28\%\to31.65\%$, DiffusionDrive: $40.51\%\to3.80\%$), demonstrating that the lateral term is what pushes the scoring head toward laterally displaced vocabulary/anchor candidates rather than collision-inducing ones. Removing $\mathcal{L}_{\text{brake}}$ occasionally yields marginally higher collision rates on individual models (e.g., GTRS-DP: $70.89\%$ vs.\ $66.46\%$ full), reflecting a mild tension between velocity-discontinuity induction and collision rate. The full loss consistently achieves the best \emph{joint} degradation across both CR and OR, confirming that the three task-directed terms address complementary axes of the PDM Score and that their combination produces the most comprehensive safety-critical degradation.

\input{tables/ablation_loss}

\input{tables/params}

\begin{figure*}[h]
\centering
\includegraphics[width=1.\textwidth]{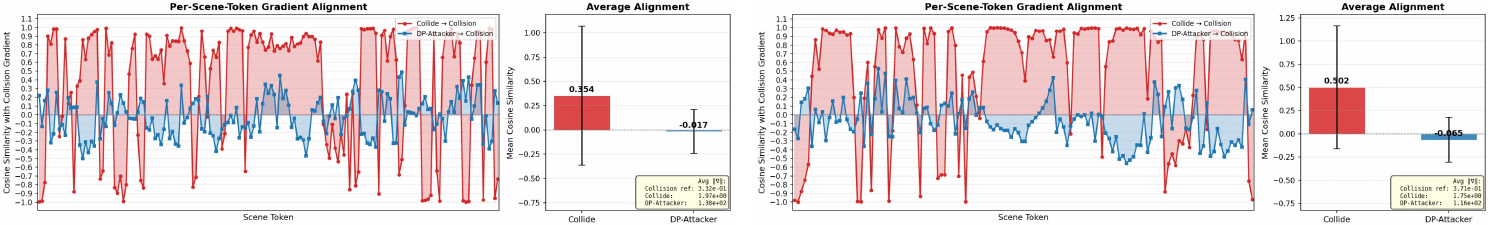}
\vspace{-0.5cm}
\caption{Gradient alignment with the collision objective across scene tokens (02 logs).}
\label{fig:gradient_alignment}
\vspace{-0.3cm}
\end{figure*}

\vspace{4pt}
\noindent \textbf{Parameters Sensitivity Analysis.}
Table~\ref{tab:pgd_sensitivity} reports the sensitivity of \textsc{Derail} to two key optimization parameters: the number of PGD iterations and the perturbation budget $\epsilon$, evaluated on GTRS-AUG. \textsc{Derail} exhibits consistent monotonic improvement with increasing iterations across all safety-critical metrics, with collision rate rising from 31.65\% at 5 iterations to 60.76\% at 50 iterations. However, gains exhibit clear diminishing returns beyond 20 iterations, the largest single improvement occurs between 5 and 10 iterations (+12.65\% CR), whereas subsequent increments of 5 iterations contribute progressively smaller gains, suggesting that the attack converges to a strong adversarial solution within a moderate optimization budget. Notably, off-road rate peaks at 45.57\% at 30 iterations before slightly declining at 50, indicating a mild tension between maximizing frontal collisions and lateral boundary violations at higher iteration counts. With respect to perturbation budget, attack effectiveness scales consistently from $\epsilon=2/255$ to $\epsilon=8/255$, with the largest gains concentrated in this range (CR: 30.38\%$\to$50.63\%, TTC: 45.57\%$\to$72.15\%). Beyond $\epsilon=8/255$, improvements plateau for CR and TTC (54.43\% at both $\epsilon=16/255$ and $\epsilon=32/255$), while OR and WW continue to grow marginally, suggesting that the collision-inducing objective saturates at moderate budgets while lateral and directional failures remain responsive to larger perturbations. Together, these results confirm that \textsc{Derail} achieves strong attack performance within a practical optimization budget of 20 iterations and $\epsilon=8/255$, balancing effectiveness with computational efficiency.

\vspace{4pt}
\noindent \textbf{Gradient Alignment Analysis.} The gap between \textsc{Derail} and the DP-Attacker baseline~\cite{chen2024diffusion} in Table~\ref{tab:digital_results} cannot be attributed to perturbation budget or optimization strength, since both attacks share the same PGD loop, same $\epsilon$, and same number of iterations. To isolate the role of the objective, we measure the per-scene-token cosine similarity between each attack's input gradient and a reference collision-inducing gradient (the gradient of a separately constructed collision-event loss) across two representative driving logs. As shown in Fig.~\ref{fig:gradient_alignment}, \textsc{Derail} achieves consistently positive gradient alignment with the collision reference (mean $0.354$ and $0.502$ across the two logs), with the red band remaining predominantly above zero throughout the scene-token sequence. In contrast, DP-Attacker's gradients are effectively orthogonal to the collision reference (mean $-0.017$ and $-0.065$), oscillating randomly around zero without any systematic alignment. We interpret this not as evidence that the diffusion decoder is intrinsically less vulnerable than vocabulary decoders, but as direct evidence that the choice of objective, task-directed (\textsc{Derail}) versus generic loss maximization (DP-Attacker), is what determines whether adversarial gradients steer the scoring head in Eq.~\eqref{eq:chain} toward unsafe candidates within a fixed perturbation budget.

%% file: tables/digital_results.tex
\begin{table}[t]
\centering
\caption{Adversarial attacks comparison on NAVSIM \cite{dauner2024navsim}.}
\label{tab:digital_results}
\vspace{-0.2cm}

\definecolor{oursaccent}{HTML}{8B2635}
\definecolor{oursrow}{HTML}{FDECEA}
\definecolor{modelband}{HTML}{FFFFFF}   
\definecolor{modelbandtext}{HTML}{000000} 

\setlength{\tabcolsep}{2.5pt}
\renewcommand{\arraystretch}{.60}

\newcommand{\modelband}[2]{%
  \multicolumn{9}{c}{\cellcolor{modelband}%
    \color{modelbandtext}\textit{Attacked Model:}~\textbf{#1}~\cite{#2}} \\
  \arrayrulecolor{modelbandtext}\midrule\arrayrulecolor{black}%
}

\resizebox{\columnwidth}{!}{%
\footnotesize
\begin{tabular}{@{}l*{8}{c}@{}}
\toprule
\textbf{Adv. Attack}
 & \textbf{PDM}$\downarrow$
 & \textbf{Drop}$\uparrow$
 & \textbf{ASR}$\uparrow$
 & \textbf{CR}$\uparrow$
 & \textbf{OR}$\uparrow$
 & \textbf{WW}$\uparrow$
 & \textbf{TTC}$\uparrow$
 & \textbf{CM}$\uparrow$ \\
\midrule

\modelband{DiffusionDrive}{dauner2024navsim}
Clean              & 0.91 & 0.00  & 0.00   & 0.99  & 2.01  & 2.24  & 4.33  & 0.01 \\
Untargeted PGD     & 0.77 & 15.86 & 95.71  & 9.14  & 6.74  & 4.30  & 21.50 & 0.01 \\
Encoder Attack     & 0.92 & -0.12 & 44.29  & 1.02  & 2.07  & 2.23  & 4.00  & 0.01 \\
DP Attacker~\cite{chen2024diffusion} & 0.60 & 34.19 & 100.00 & 21.01 & 10.14 & 8.21  & 31.35 & 0.23 \\
\rowcolor{oursrow}
\textbf{Derail (Ours)}
 & \textbf{\textcolor{oursaccent}{0.37}}
 & \textbf{\textcolor{oursaccent}{60.02}}
 & \textbf{\textcolor{oursaccent}{100.00}}
 & \textbf{\textcolor{oursaccent}{38.42}}
 & \textbf{\textcolor{oursaccent}{31.95}}
 & \textbf{\textcolor{oursaccent}{22.63}}
 & \textbf{\textcolor{oursaccent}{52.67}}
 & \textbf{\textcolor{oursaccent}{3.79}} \\
\cmidrule(l){1-9}

\modelband{GTRS-DP}{li2025generalized}
Clean              & 0.78 & 0.00  & 0.00   & 11.28 & 3.51  & 0.57  & 12.12 & 0.93 \\
Untargeted PGD     & 0.42 & 46.03 & 98.57  & 29.80 & 27.15 & 13.85 & 33.20 & 5.80 \\
Encoder Attack     & 0.62 & 20.44 & 90.00  & 15.98 & 15.86 & 4.81  & 17.68 & 0.91 \\
DP Attacker~\cite{chen2024diffusion} & 0.59 & 24.26 & 94.29  & 17.87 & 19.22 & 6.99  & 19.26 & 0.90 \\
\rowcolor{oursrow}
\textbf{Derail (Ours)}
 & \textbf{\textcolor{oursaccent}{0.16}}
 & \textbf{\textcolor{oursaccent}{80.10}}
 & \textbf{\textcolor{oursaccent}{100.00}}
 & \textbf{\textcolor{oursaccent}{50.65}}
 & \textbf{\textcolor{oursaccent}{60.38}}
 & \textbf{\textcolor{oursaccent}{42.24}}
 & \textbf{\textcolor{oursaccent}{58.18}}
 & \textbf{\textcolor{oursaccent}{18.90}} \\
\cmidrule(l){1-9}

\modelband{GTRS-AUG}{li2025generalized}
Clean              & 0.83 & 0.00  & 0.00   & 9.51  & 0.40  & 0.30  & 9.99  & 0.80 \\
Untargeted PGD     & 0.46 & 45.03 & 98.57  & 30.72 & 20.70 & 8.46  & 32.55 & 1.54 \\
Encoder Attack     & 0.64 & 23.09 & 92.86  & 11.19 & 17.01 & 6.15  & 13.47 & 3.89 \\
DP Attacker~\cite{chen2024diffusion} & 0.46 & 44.74 & 98.57  & 30.83 & 20.25 & 8.75  & 33.07 & 1.61 \\
\rowcolor{oursrow}
\textbf{Derail (Ours)}
 & \textbf{\textcolor{oursaccent}{0.22}}
 & \textbf{\textcolor{oursaccent}{72.85}}
 & \textbf{\textcolor{oursaccent}{100.00}}
 & \textbf{\textcolor{oursaccent}{45.04}}
 & \textbf{\textcolor{oursaccent}{47.92}}
 & \textbf{\textcolor{oursaccent}{33.66}}
 & \textbf{\textcolor{oursaccent}{51.41}}
 & \textbf{\textcolor{oursaccent}{8.19}} \\
\cmidrule(l){1-9}

\modelband{GTRS-Dense}{li2025generalized}
Clean              & 0.83 & 0.00  & 0.00   & 7.90  & 1.00  & 0.29  & 8.22  & 0.88 \\
Untargeted PGD     & 0.69 & 16.82 & 87.14  & 8.01  & 10.70 & 3.56  & 8.83  & 4.85 \\
Encoder Attack     & 0.67 & 19.07 & 90.00  & 5.66  & 10.04 & 2.31  & 7.04  & 7.52 \\
DP Attacker~\cite{chen2024diffusion} & 0.69 & 16.42 & 85.71  & 7.54  & 10.73 & 3.90  & 8.28  & 4.17 \\
\rowcolor{oursrow}
\textbf{Derail (Ours)}
 & \textbf{\textcolor{oursaccent}{0.50}}
 & \textbf{\textcolor{oursaccent}{39.78}}
 & \textbf{\textcolor{oursaccent}{100.00}}
 & \textbf{\textcolor{oursaccent}{25.82}}
 & \textbf{\textcolor{oursaccent}{10.49}}
 & \textbf{\textcolor{oursaccent}{4.10}}
 & \textbf{\textcolor{oursaccent}{27.45}}
 & \textbf{\textcolor{oursaccent}{43.25}} \\

\bottomrule
\end{tabular}}
\vspace{-0.5cm}
\end{table}

%% file: tables/one_log_patch.tex
\begin{table}[t]
\centering
\caption{Patch attack on one driving log from NAVSIM \cite{dauner2024navsim}.}
\label{tab:one_log_patch_results}
\vspace{-0.2cm}

\definecolor{oursaccent}{HTML}{8B2635}
\definecolor{oursrow}{HTML}{FDECEA}
\definecolor{modelband}{HTML}{FFFFFF}
\definecolor{modelbandtext}{HTML}{000000}

\setlength{\tabcolsep}{2.5pt}
\renewcommand{\arraystretch}{.60}

\newcommand{\modelband}[2]{%
  \multicolumn{8}{c}{\cellcolor{modelband}%
    \color{modelbandtext}\textit{Attacked Model:}~\textbf{#1}~\cite{#2}} \\
  \arrayrulecolor{modelbandtext}\midrule\arrayrulecolor{black}%
}

\resizebox{\columnwidth}{!}{%
\footnotesize
\begin{tabular}{@{}l*{7}{c}@{}}
\toprule
\textbf{Adv. Attack}
 & \textbf{PDM}$\downarrow$
 & \textbf{Drop}$\uparrow$
 & \textbf{ASR}$\uparrow$
 & \textbf{CR}$\uparrow$
 & \textbf{OR}$\uparrow$
 & \textbf{WW}$\uparrow$
 & \textbf{TTC}$\uparrow$ \\
\midrule

\modelband{DiffusionDrive}{dauner2024navsim}
Clean         & 0.81 & 0.00  & 0.00   & 16.67 & 0.00 & 0.00 & 16.67 \\
Random Patch  & 0.73 & 9.46  & 100.00 & 25.00 & 0.00 & 0.00 & 25.00 \\
\rowcolor{oursrow}
\textbf{Derail Patch (Ours)}
 & \textbf{\textcolor{oursaccent}{0.66}}
 & \textbf{\textcolor{oursaccent}{18.49}}
 & \textbf{\textcolor{oursaccent}{100.00}}
 & \textbf{\textcolor{oursaccent}{33.33}}
 & \textbf{\textcolor{oursaccent}{0.00}}
 & \textbf{\textcolor{oursaccent}{0.00}}
 & \textbf{\textcolor{oursaccent}{33.33}} \\
\cmidrule(l){1-8}

\modelband{GTRS-DP}{li2025generalized}
Clean         & 0.84 & 0.00  & 0.00   & 8.33  & 0.00 & 0.00 & 16.67 \\
Random Patch  & 0.87 & -3.51 & 0.00   & 8.33  & 0.00 & 0.00 & 8.33  \\
\rowcolor{oursrow}
\textbf{Derail Patch (Ours)}
 & \textbf{\textcolor{oursaccent}{0.71}}
 & \textbf{\textcolor{oursaccent}{15.22}}
 & \textbf{\textcolor{oursaccent}{100.00}}
 & \textbf{\textcolor{oursaccent}{25.00}}
 & \textbf{\textcolor{oursaccent}{0.00}}
 & \textbf{\textcolor{oursaccent}{0.00}}
 & \textbf{\textcolor{oursaccent}{25.00}} \\
\cmidrule(l){1-8}

\modelband{GTRS-AUG}{li2025generalized}
Clean         & 0.82 & 0.00  & 0.00   & 16.67 & 0.00 & 0.00 & 16.67 \\
Random Patch  & 0.82 & 0.01  & 100.00 & 16.67 & 0.00 & 0.00 & 16.67 \\
\rowcolor{oursrow}
\textbf{Derail Patch (Ours)}
 & \textbf{\textcolor{oursaccent}{0.74}}
 & \textbf{\textcolor{oursaccent}{10.17}}
 & \textbf{\textcolor{oursaccent}{100.00}}
 & \textbf{\textcolor{oursaccent}{25.00}}
 & \textbf{\textcolor{oursaccent}{0.00}}
 & \textbf{\textcolor{oursaccent}{0.00}}
 & \textbf{\textcolor{oursaccent}{25.00}} \\
\cmidrule(l){1-8}

\modelband{GTRS-Dense}{li2025generalized}
Clean         & 0.85 & 0.00  & 0.00   & 8.33  & 0.00 & 0.00 & 8.33 \\
Random Patch  & 0.85 & -1.11 & 0.00   & 8.33  & 0.00 & 0.00 & 8.33 \\
\rowcolor{oursrow}
\textbf{Derail Patch (Ours)}
 & \textbf{\textcolor{oursaccent}{0.62}}
 & \textbf{\textcolor{oursaccent}{26.18}}
 & \textbf{\textcolor{oursaccent}{100.00}}
 & \textbf{\textcolor{oursaccent}{33.33}}
 & \textbf{\textcolor{oursaccent}{0.00}}
 & \textbf{\textcolor{oursaccent}{0.00}}
 & \textbf{\textcolor{oursaccent}{33.33}} \\

\bottomrule
\end{tabular}}
\vspace{-0.5cm}
\end{table}

%% file: tables/all_log_patch.tex
\begin{table}[t]
\centering
\caption{Patch attack comparison on NAVSIM \cite{dauner2024navsim} (all-logs).}
\label{tab:all_logs_patch_results}
\vspace{-0.2cm}

\definecolor{oursaccent}{HTML}{8B2635}
\definecolor{oursrow}{HTML}{FDECEA}
\definecolor{modelband}{HTML}{FFFFFF}
\definecolor{modelbandtext}{HTML}{000000}

\setlength{\tabcolsep}{2.5pt}
\renewcommand{\arraystretch}{.60}

\newcommand{\modelband}[2]{%
  \multicolumn{8}{c}{\cellcolor{modelband}%
    \color{modelbandtext}\textit{Attacked Model:}~\textbf{#1}~\cite{#2}} \\
  \arrayrulecolor{modelbandtext}\midrule\arrayrulecolor{black}%
}

\resizebox{\columnwidth}{!}{%
\footnotesize
\begin{tabular}{@{}l*{7}{c}@{}}
\toprule
\textbf{Adv. Attack}
 & \textbf{PDM}$\downarrow$
 & \textbf{Drop}$\uparrow$
 & \textbf{ASR}$\uparrow$
 & \textbf{CR}$\uparrow$
 & \textbf{OR}$\uparrow$
 & \textbf{WW}$\uparrow$
 & \textbf{TTC}$\uparrow$ \\
\midrule

\modelband{DiffusionDrive}{dauner2024navsim}
Clean         & 0.91 & 0.00  & 0.00  & 1.30  & 1.86 & 1.96 & 4.89  \\
Random Patch  & 0.90 & 0.50  & 44.78 & 2.15  & 0.94 & 0.63 & 7.07  \\
\rowcolor{oursrow}
\textbf{Derail Patch (Ours)}
 & \textbf{\textcolor{oursaccent}{0.79}}
 & \textbf{\textcolor{oursaccent}{13.38}}
 & \textbf{\textcolor{oursaccent}{68.66}}
 & \textbf{\textcolor{oursaccent}{13.56}}
 & \textbf{\textcolor{oursaccent}{1.17}}
 & \textbf{\textcolor{oursaccent}{1.15}}
 & \textbf{\textcolor{oursaccent}{22.80}} \\
\cmidrule(l){1-8}

\modelband{GTRS-DP}{li2025generalized}
Clean         & 0.80 & 0.00  & 0.00  & 9.34  & 1.15 & 0.62 & 9.49  \\
Random Patch  & 0.81 & -0.09 & 55.22 & 9.24  & 1.13 & 0.65 & 9.26  \\
\rowcolor{oursrow}
\textbf{Derail Patch (Ours)}
 & \textbf{\textcolor{oursaccent}{0.78}}
 & \textbf{\textcolor{oursaccent}{2.91}}
 & \textbf{\textcolor{oursaccent}{61.19}}
 & \textbf{\textcolor{oursaccent}{10.68}}
 & \textbf{\textcolor{oursaccent}{2.11}}
 & \textbf{\textcolor{oursaccent}{0.64}}
 & \textbf{\textcolor{oursaccent}{10.86}} \\
\cmidrule(l){1-8}

\modelband{GTRS-AUG}{li2025generalized}
Clean         & 0.93 & 0.00  & 0.00  & 2.59  & 0.36 & 0.11 & 2.66  \\
Random Patch  & 0.93 & 0.39  & 47.76 & 2.60  & 1.00 & 0.10 & 2.56  \\
\rowcolor{oursrow}
\textbf{Derail Patch (Ours)}
 & \textbf{\textcolor{oursaccent}{0.86}}
 & \textbf{\textcolor{oursaccent}{7.37}}
 & \textbf{\textcolor{oursaccent}{50.75}}
 & \textbf{\textcolor{oursaccent}{9.12}}
 & \textbf{\textcolor{oursaccent}{1.47}}
 & \textbf{\textcolor{oursaccent}{0.28}}
 & \textbf{\textcolor{oursaccent}{8.50}} \\
\cmidrule(l){1-8}

\modelband{GTRS-Dense}{li2025generalized}
Clean         & 0.91 & 0.00  & 0.00  & 0.69  & 1.22 & 0.33 & 0.83  \\
Random Patch  & 0.91 & 0.05  & 41.79 & 0.73  & 1.20 & 0.33 & 1.01  \\
\rowcolor{oursrow}
\textbf{Derail Patch (Ours)}
 & \textbf{\textcolor{oursaccent}{0.88}}
 & \textbf{\textcolor{oursaccent}{3.29}}
 & \textbf{\textcolor{oursaccent}{52.24}}
 & \textbf{\textcolor{oursaccent}{2.90}}
 & \textbf{\textcolor{oursaccent}{1.50}}
 & \textbf{\textcolor{oursaccent}{0.35}}
 & \textbf{\textcolor{oursaccent}{2.64}} \\

\bottomrule
\end{tabular}}
\vspace{-0.6cm}
\end{table}

%% file: tables/ablation_loss.tex
\begin{table}[t]
\centering
\caption{Ablation analysis on the \textsc{Derail} loss components.}
\label{tab:ablation_loss}
\vspace{-0.2cm}

\definecolor{oursaccent}{HTML}{8B2635}
\definecolor{oursrow}{HTML}{FDECEA}

\setlength{\tabcolsep}{2.4pt}
\renewcommand{\arraystretch}{1.2}

\resizebox{\columnwidth}{!}{%
\footnotesize
\begin{tabular}{@{}l*{12}{c}@{}}
\toprule
\multirow{2}{*}{\textbf{Adv. Loss}}
 & \multicolumn{3}{c}{\textbf{DiffusionDrive}~\cite{liao2025diffusiondrive}}
 & \multicolumn{3}{c}{\textbf{GTRS-DP}~\cite{li2025generalized}}
 & \multicolumn{3}{c}{\textbf{GTRS-AUG}~\cite{li2025generalized}}
 & \multicolumn{3}{c}{\textbf{GTRS-Dense}~\cite{li2025generalized}} \\
\cmidrule(lr){2-4} \cmidrule(lr){5-7} \cmidrule(lr){8-10} \cmidrule(lr){11-13}
 & PDM$\downarrow$ & CR$\uparrow$ & OR$\uparrow$
 & PDM$\downarrow$ & CR$\uparrow$ & OR$\uparrow$
 & PDM$\downarrow$ & CR$\uparrow$ & OR$\uparrow$
 & PDM$\downarrow$ & CR$\uparrow$ & OR$\uparrow$ \\
\midrule

Clean
 & 0.92 & 2.53  & 0.00
 & 0.87 & 5.06  & 0.00
 & 0.84 & 11.39 & 0.00
 & 0.86 & 6.33  & 0.00 \\

 w/o $\mathcal{L}_{\text{aggr}}$
 & 0.31 & 40.51 & 37.97
 & 0.32 & 53.16 & 53.16
 & 0.38 & 44.30 & 32.91
 & 0.60 & 17.72 & 13.92 \\

 w/o $\mathcal{L}_{\text{bound}}$
 & 0.42 & 46.84 & 3.80
 & 0.22 & 59.49 & 31.65
 & 0.32 & 58.86 & 15.19
 & 0.42 & 43.04 & 3.80 \\

 w/o $\mathcal{L}_{\text{brake}}$
 & 0.15 & 63.92 & 36.71
 & 0.05 & 70.89 & 77.22
 & 0.23 & 55.70 & 27.85
 & 0.32 & 56.96 & 3.80 \\

\rowcolor{oursrow}
\textbf{Full $\mathcal{L}_{\text{Derail}}$}
 & \textbf{\textcolor{oursaccent}{0.19}}
 & \textbf{\textcolor{oursaccent}{63.29}}
 & \textbf{\textcolor{oursaccent}{40.51}}
 & \textbf{\textcolor{oursaccent}{0.05}}
 & \textbf{\textcolor{oursaccent}{66.46}}
 & \textbf{\textcolor{oursaccent}{82.28}}
 & \textbf{\textcolor{oursaccent}{0.22}}
 & \textbf{\textcolor{oursaccent}{60.76}}
 & \textbf{\textcolor{oursaccent}{29.11}}
 & \textbf{\textcolor{oursaccent}{0.38}}
 & \textbf{\textcolor{oursaccent}{48.73}}
 & \textbf{\textcolor{oursaccent}{5.06}} \\

\bottomrule
\end{tabular}}
\vspace{-0.5cm}
\end{table}

%% file: tables/params.tex
\begin{table*}[t]
\centering
\caption{Sensitivity analysis of \textsc{derail} attack strength to PGD optimization iterations or perturbation budget.}
\label{tab:pgd_sensitivity}
\vspace{-0.2cm}

\definecolor{oursaccent}{HTML}{8B2635}

\setlength{\tabcolsep}{2.5pt}
\renewcommand{\arraystretch}{.5}

\footnotesize
\begin{tabular}{@{}l*{7}{c}@{\hspace{18pt}}l*{5}{c}@{}}
\toprule
\multirow{2}{*}{\textbf{Metric}}
 & \multicolumn{7}{c}{\textcolor{oursaccent}{\textbf{PGD Iterations}}~($\epsilon = 8/255$)}
 & \multirow{2}{*}{\textbf{Metric}}
 & \multicolumn{5}{c}{\textcolor{oursaccent}{\textbf{Perturbation Budget}}~($\alpha = 2/255$)} \\
\cmidrule(lr){2-8} \cmidrule(lr){10-14}
 & 5 & 10 & 15 & 20 & 25 & 30 & 50
 & & $\epsilon = 2/255$ & $\epsilon = 4/255$ & $\epsilon = 8/255$ & $\epsilon = 16/255$ & $\epsilon = 32/255$ \\
\midrule
PDM$\downarrow$ & 0.45  & 0.30  & 0.29  & 0.25  & 0.23  & 0.21  & 0.20
              & PDM$\downarrow$ & 0.48 & 0.31 & 0.25 & 0.23 & 0.22 \\
CR$\uparrow$    & 31.65 & 44.30 & 48.10 & 50.63 & 51.90 & 56.96 & 60.76
              & CR$\uparrow$    & 30.38 & 44.94 & 50.63 & 54.43 & 54.43 \\
OR$\uparrow$    & 29.11 & 40.51 & 41.77 & 43.04 & 44.30 & 45.57 & 44.30
              & OR$\uparrow$    & 25.32 & 39.24 & 43.04 & 45.57 & 46.84 \\
WW$\uparrow$    & 4.43  & 13.29 & 16.46 & 17.09 & 23.42 & 23.42 & 26.58
              & WW$\uparrow$    & 3.16  & 10.13 & 17.09 & 21.52 & 23.42 \\
TTC$\uparrow$   & 51.90 & 64.56 & 68.35 & 72.15 & 75.95 & 74.68 & 82.28
              & TTC$\uparrow$   & 45.57 & 65.82 & 72.15 & 78.48 & 78.48 \\
\bottomrule
\end{tabular}
\end{table*}

%% file: sec/5_conclusion.tex
\section{Discussion}

\textsc{Derail} demonstrated that current generative end-to-end autonomous driving planners share a design-enabled adversarial attack surface that follows directly from how they assemble their final motion command. Across diffusion-based (DiffusionDrive~\cite{liao2025diffusiondrive}) and vocabulary-based (GTRS variants~\cite{li2025generalized}) planners, inference reduces to the same five-stage chain of Eq.~\eqref{eq:chain}: a learned scoring head over a fixed candidate set, with no post-hoc safety filter at the model level. Because the scoring head is the only barrier between the BEV features and the selected trajectory, adversarial perturbations to the camera input can switch the scoring head's selection from a safe to an unsafe trajectory.

\section{Conclusion}
This paper presents \textsc{Derail}, an adversarial framework that targets a recurring scoring-head attack surface in current generative E2E AD planners. Across four state-of-the-art models spanning diffusion-based and vocabulary-based trajectory decoding on NAVSIM, \textsc{Derail} achieves $100\%$ ASR on every model with collision rates of $25.8$--$60.4\%$. Our analysis suggests that safety-violating objectives govern attack effectiveness against this class of planners and that explicit downstream safety filtering between the scoring head and the executed trajectory deserves attention as a defensive direction for next-generation of generative end-to-end autonomous driving systems.

%% file: sec/X_supp.tex
\maketitlesupplementary
\setcounter{page}{1}

\startcontents

\begin{strip}
    \vspace{-3em} 
    \centering
    \section*{Supplementary Material Overview}
    \vspace{-1em} 
    
    \begin{minipage}{.95\textwidth} 
        \printcontents{}{1}{\textbf{Contents}\vskip3pt\hrule\vskip5pt}
        \vskip3pt\hrule\vskip5pt
    \end{minipage}
    \vspace{1.5em} 
\end{strip}

\input{appendix/rationale}

\input{appendix/details}

\input{appendix/additional_results}

%% file: appendix/rationale.tex
\section{Why \textsc{Derail} can Break Generative AD?}
\label{sec:vulnerability_modeling}

This appendix provides the analytic framework that supports the architectural argument of the main paper: that current generative AD planners share a scoring-head inference chain (Eq.~\eqref{eq:chain}, main paper) whose decision margins are small enough that pixel-space and patch-based perturbations can flip the selected trajectory from a safe to an unsafe candidate. We instantiate this argument for the two representative decoder families present in the planners we evaluate, namely diffusion-based refinement of a fixed anchor set (DiffusionDrive \cite{liao2025diffusiondrive}) and vocabulary-based scoring over a fixed trajectory set (GTRS variants \cite{li2025generalized}), and we show that, in both cases, the trajectory deviation under an adversarial perturbation can be bounded by a product of an encoder sensitivity term and a decoder amplification term that arises directly from the iterative structure of the chain.

\noindent\textbf{Scope of the claim.} The bounds below are intentionally architecture-specific: they apply to the decoder families we actually evaluate, and they should be read as describing a \emph{design-enabled} attack surface. Any planner whose decoder unrolls into $K{\geq}1$ feature-conditioned steps before the scoring head's final selection, and that does not apply a non-learned downstream safety filter, inherits the same multiplicative amplification structure; whether other decoder families exhibit it in practice is an empirical question we do not attempt to settle here.

\subsection{The Shared Inference Chain}
\label{sec:vulnerability_taxonomy}

We summarize three properties of the chain in Eq.~\eqref{eq:chain} of the main paper that expose the scoring-head attack surface:

\begin{enumerate}
    \item \textbf{Feature-conditioned scoring}: The per-candidate score $\ell_v$ at the head of the chain is computed by a learned neural network conditioned on the BEV features $\mathbf{Z} = \mathcal{E}_{\theta}(\mathbf{I})$. The chain therefore has a single differentiable bottleneck through which adversarial influence propagates.
    \item \textbf{Iterative amplification of the conditioning signal}: When the decoder applies the conditioning features $\mathbf{Z}$ across $K{\geq}1$ sequential refinement steps (denoising iterations in DiffusionDrive, denoising in the diffusion-policy head of GTRS-DP, or vocabulary scoring rounds in GTRS-Aug/Dense), small feature perturbations compound multiplicatively along the chain.
    \item \textbf{Small candidate-margin to dangerous trajectories}: The fixed candidate set contains many trajectories that lie close in score, including candidates that the PDM Score would label as unsafe. Switching the scoring head's selection from a safe to an unsafe candidate only requires shifting the relative score across a small margin.
\end{enumerate}

\noindent We formalize the first two properties as Lipschitz-style bounds on $\|\Delta \mathbf{Z}\|_F$ and on the trajectory deviation, and the third as a margin condition on the scoring head; the loss-alignment results in Sec.~\ref{sec:loss_alignment} then explain why the safety-violating \textsc{Derail} objectives is matched to these bounds.

\subsection{Encoder Sensitivity and Feature Perturbation}
\label{sec:encoder_sensitivity}

Let $\mathcal{E}_\theta : \mathbb{R}^{C \times 3 \times H_I \times W_I} \to \mathbb{R}^{d \times H \times W}$ denote the visual encoder. We begin by characterizing how input-space perturbations translate to feature-space deviations.

\begin{definition}[Encoder Lipschitz Constant]
\label{def:lipschitz}
The encoder $\mathcal{E}_\theta$ is \emph{Lipschitz continuous} with constant $L_{\mathcal{E}}$ if for all input pairs $\mathbf{I}, \tilde{\mathbf{I}}$:
\begin{equation}
    \|\mathcal{E}_\theta(\tilde{\mathbf{I}}) - \mathcal{E}_\theta(\mathbf{I})\|_F \leq L_{\mathcal{E}} \cdot \|\tilde{\mathbf{I}} - \mathbf{I}\|_\infty,
    \label{eq:encoder_lipschitz}
\end{equation}
where $\|\cdot\|_F$ denotes the Frobenius norm on the feature tensor and $\|\cdot\|_\infty$ is the pixel-space $\ell_\infty$ norm.
\end{definition}

\begin{proposition}[Feature Perturbation Bound]
\label{prop:feature_bound}
Under the digital attack model (Eq. 5 main paper) with budget $\epsilon$, the adversarial feature deviation satisfies:
\begin{equation}
    \|\Delta \mathbf{Z}\|_F \triangleq \|\mathcal{E}_\theta(\tilde{\mathbf{I}}) - \mathcal{E}_\theta(\mathbf{I})\|_F \leq L_{\mathcal{E}} \cdot \epsilon.
    \label{eq:feature_perturbation_bound}
\end{equation}
\end{proposition}

\begin{proof}
Direct application of Definition~\ref{def:lipschitz} with $\|\tilde{\mathbf{I}} - \mathbf{I}\|_\infty = \|\boldsymbol{\delta}\|_\infty \leq \epsilon$.
\end{proof}

\noindent \textbf{Remark.} While the bound~\eqref{eq:feature_perturbation_bound} is linear in $\epsilon$, the \emph{effective} Lipschitz constant $L_{\mathcal{E}}$ of modern convolutional and transformer-based encoders is typically very large due to the composition of many layers with ReLU or GELU activations. Empirically, $L_{\mathcal{E}}$ scales exponentially with network depth for unconstrained architectures~\cite{szegedy2013intriguing, virmaux2018lipschitz}. This means even a small $\epsilon$ (e.g., $8/255 \approx 0.031$) can induce feature deviations $\|\Delta \mathbf{Z}\|_F$ that are orders of magnitude larger than the pixel perturbation norm.

\noindent \textbf{Local Sensitivity via the Jacobian.}
For a more precise characterization, consider the encoder Jacobian $\mathbf{J}_{\mathcal{E}} = \frac{\partial \mathcal{E}_\theta(\mathbf{I})}{\partial \mathbf{I}} \in \mathbb{R}^{(dHW) \times (3CH_IW_I)}$. By first-order Taylor expansion:
\begin{equation}
    \Delta \mathbf{Z} \approx \mathbf{J}_{\mathcal{E}} \cdot \boldsymbol{\delta}, \quad \|\Delta \mathbf{Z}\|_F \approx \|\mathbf{J}_{\mathcal{E}} \cdot \boldsymbol{\delta}\|_2.
    \label{eq:jacobian_approx}
\end{equation}
The adversarial perturbation $\boldsymbol{\delta}^* = \epsilon \cdot \text{sign}(\mathbf{J}_{\mathcal{E}}^\top \mathbf{v})$, where $\mathbf{v}$ is aligned with the attack objective gradient in feature space, maximizes the projection of $\Delta \mathbf{Z}$ along the direction most harmful to the downstream trajectory decoder. This is precisely what PGD computes through end-to-end backpropagation.

\subsection{Perturbation Amplification in Decoders}
\label{sec:amplification}

The key insight enabling the \textsc{Derail} attack is that generative trajectory decoders do not merely propagate feature perturbations, they \emph{amplify} them through iterative generation process. We formalize this for each architectural paradigm.

\subsubsection{Diffusion-Based Trajectory Planning}
\label{sec:diffusion_vulnerability}

Diffusion-based planning models (e.g., DiffusionDrive~\cite{liao2025diffusiondrive}) decode trajectories through a learned reverse diffusion process. Starting from Gaussian noise $\boldsymbol{\tau}_K \sim \mathcal{N}(\mathbf{0}, \mathbf{I})$, the decoder iteratively denoises over $K$ steps:
\begin{equation}
    \boldsymbol{\tau}_{k-1} = \frac{1}{\sqrt{\alpha_k}}\left(\boldsymbol{\tau}_k - \frac{1 - \alpha_k}{\sqrt{1 - \bar{\alpha}_k}} \boldsymbol{\epsilon}_\theta(\boldsymbol{\tau}_k, k, \mathbf{Z})\right) + \sigma_k \mathbf{z}_k,
    \label{eq:ddpm_reverse}
\end{equation}
where $\boldsymbol{\epsilon}_\theta(\cdot, k, \mathbf{Z})$ is the learned noise predictor conditioned on features $\mathbf{Z}$, $\alpha_k, \bar{\alpha}_k$ are the noise schedule parameters, and $\mathbf{z}_k \sim \mathcal{N}(\mathbf{0}, \mathbf{I})$.

\begin{theorem}[Diffusion Perturbation Amplification]
\label{thm:diffusion_amplification}
Let $\boldsymbol{\tau}_0$ and $\tilde{\boldsymbol{\tau}}_0$ denote the final trajectories produced under clean features $\mathbf{Z}$ and adversarial features $\tilde{\mathbf{Z}} = \mathbf{Z} + \Delta \mathbf{Z}$, respectively, starting from the same noise realization $\boldsymbol{\tau}_K$. If the noise predictor satisfies $\|\boldsymbol{\epsilon}_\theta(\boldsymbol{\tau}, k, \tilde{\mathbf{Z}}) - \boldsymbol{\epsilon}_\theta(\boldsymbol{\tau}, k, \mathbf{Z})\|_2 \leq L_\epsilon \|\Delta \mathbf{Z}\|_F$ for all $k$ and $\boldsymbol{\tau}$, then the trajectory deviation satisfies:
\begin{equation}
\begin{aligned}
    \|\tilde{\boldsymbol{\tau}}_0 - \boldsymbol{\tau}_0\|_2 \leq{} & L_\epsilon \|\Delta \mathbf{Z}\|_F \cdot \sum_{k=1}^{K} \frac{1 - \alpha_k}{\sqrt{\alpha_k(1 - \bar{\alpha}_k)}} \\
    & \cdot \prod_{j=1}^{k-1} \left(\frac{1}{\sqrt{\alpha_j}} + \frac{(1 - \alpha_j) L_{\boldsymbol{\tau}}}{\sqrt{\alpha_j(1 - \bar{\alpha}_j)}}\right),
    \label{eq:diffusion_bound}
\end{aligned}
\end{equation}
where $L_{\boldsymbol{\tau}}$ is the Lipschitz constant of $\boldsymbol{\epsilon}_\theta$ with respect to $\boldsymbol{\tau}$.
\end{theorem}

\begin{proof}
Define $\Delta_k \triangleq \tilde{\boldsymbol{\tau}}_k - \boldsymbol{\tau}_k$ as the trajectory deviation at step $k$. At the initial noise step $K$, $\Delta_K = \mathbf{0}$ (same noise realization). From the reverse process~\eqref{eq:ddpm_reverse}, the deviation at step $k-1$ is:
\begin{align}
    \Delta_{k-1} ={} & \frac{1}{\sqrt{\alpha_k}} \Delta_k \nonumber \\
    & - \frac{1 - \alpha_k}{\sqrt{\alpha_k(1 - \bar{\alpha}_k)}} \left[\boldsymbol{\epsilon}_\theta(\tilde{\boldsymbol{\tau}}_k, k, \tilde{\mathbf{Z}}) - \boldsymbol{\epsilon}_\theta(\boldsymbol{\tau}_k, k, \mathbf{Z})\right]. \label{eq:diff_deviation_step}
\end{align}
Decomposing the noise predictor difference via the triangle inequality:
\begin{align}
    &\|\boldsymbol{\epsilon}_\theta(\tilde{\boldsymbol{\tau}}_k, k, \tilde{\mathbf{Z}}) - \boldsymbol{\epsilon}_\theta(\boldsymbol{\tau}_k, k, \mathbf{Z})\|_2 \nonumber \\
    &\quad \leq \underbrace{\|\boldsymbol{\epsilon}_\theta(\tilde{\boldsymbol{\tau}}_k, k, \tilde{\mathbf{Z}}) - \boldsymbol{\epsilon}_\theta(\boldsymbol{\tau}_k, k, \tilde{\mathbf{Z}})\|_2}_{\leq\, L_{\boldsymbol{\tau}} \|\Delta_k\|_2 \text{ (Lipschitz in $\boldsymbol{\tau}$)}} \nonumber \\
    &\quad\quad + \underbrace{\|\boldsymbol{\epsilon}_\theta(\boldsymbol{\tau}_k, k, \tilde{\mathbf{Z}}) - \boldsymbol{\epsilon}_\theta(\boldsymbol{\tau}_k, k, \mathbf{Z})\|_2}_{\leq\, L_\epsilon \|\Delta \mathbf{Z}\|_F \text{ (Lipschitz in $\mathbf{Z}$)}}. \label{eq:triangle_decomposition}
\end{align}
Taking norms on both sides of~\eqref{eq:diff_deviation_step}:
\begin{equation}
\begin{aligned}
    \|\Delta_{k-1}\|_2 \leq{} & \underbrace{\left(\frac{1}{\sqrt{\alpha_k}} + \frac{(1-\alpha_k)L_{\boldsymbol{\tau}}}{\sqrt{\alpha_k(1-\bar{\alpha}_k)}}\right)}_{A_k} \|\Delta_k\|_2 \\
    & + \underbrace{\frac{(1-\alpha_k)L_\epsilon}{\sqrt{\alpha_k(1-\bar{\alpha}_k)}}}_{B_k} \|\Delta \mathbf{Z}\|_F.
    \label{eq:deviation_recursion}
\end{aligned}
\end{equation}
This is a linear recurrence $\|\Delta_{k-1}\|_2 \leq A_k \|\Delta_k\|_2 + B_k \|\Delta \mathbf{Z}\|_F$ with $\|\Delta_K\|_2 = 0$. Unrolling from $k = K$ to $k = 1$ yields:
\begin{equation}
    \|\Delta_0\|_2 \leq \|\Delta \mathbf{Z}\|_F \cdot \sum_{k=1}^{K} B_k \prod_{j=1}^{k-1} A_j,
    \label{eq:unrolled_bound}
\end{equation}
which is~\eqref{eq:diffusion_bound}. Substituting $\|\Delta \mathbf{Z}\|_F \leq L_{\mathcal{E}} \epsilon$ (Proposition~\ref{prop:feature_bound}) gives the end-to-end bound in terms of the pixel budget $\epsilon$.
\end{proof}

\noindent \textbf{Interpretation.} The bound~\eqref{eq:diffusion_bound} reveals a critical structural vulnerability: the trajectory deviation grows as a \emph{product of per-step amplification factors} $A_k$, summed over all denoising steps. For typical noise schedules where early steps have small $\alpha_k$ (high noise), the amplification factors $A_k$ are large. This means that even a modest feature deviation $\|\Delta \mathbf{Z}\|_F$ is amplified multiplicatively through the reverse chain.

\begin{corollary}[Exponential Amplification Regime]
\label{cor:exponential_diffusion}
If the per-step amplification factor satisfies $A_k \geq 1 + \eta$ for some $\eta > 0$ across all $k$, then:
\begin{equation}
    \|\tilde{\boldsymbol{\tau}}_0 - \boldsymbol{\tau}_0\|_2 = \mathcal{O}\!\left(L_\epsilon \|\Delta \mathbf{Z}\|_F \cdot K \cdot (1+\eta)^K\right),
    \label{eq:exponential_amplification}
\end{equation}
i.e., the trajectory deviation grows \emph{exponentially} with the number of denoising steps $K$.
\end{corollary}

\begin{proof}
When $A_k \geq 1 + \eta$ for all $k$, each inner product in~\eqref{eq:unrolled_bound} satisfies $\prod_{j=1}^{k-1} A_j \geq (1+\eta)^{k-1}$. Summing $\sum_{k=1}^{K} (1+\eta)^{k-1} = \frac{(1+\eta)^K - 1}{\eta} = \mathcal{O}((1+\eta)^K)$ gives the result.
\end{proof}

\noindent In DiffusionDrive~\cite{liao2025diffusiondrive}, the truncated diffusion schedule reduced the denoising steps from $K = 20$ to $K = 2$ with anchor-based initialization for efficiency. While the anchor initialization reduces the effective noise level (and thus the amplification in the earliest steps), the multiplicative compounding still produces substantial amplification. Even with $K = 20$ denoising steps, we observe trajectory deviations of $2$--$5$ meters from pixel perturbations of $\epsilon = 8/255$.

\subsubsection{Vocabulary-Based Trajectory Planning}
\label{sec:vocab_vulnerability}

Vocabulary-based planning models (e.g., GTRS~\cite{li2025generalized}) maintain a fixed set of $V$ prototype trajectories $\{\boldsymbol{\tau}^{(v)}\}_{v=1}^{V}$ and select the optimal trajectory via learned scoring. Given features $\mathbf{Z}$, the model computes logits $\ell_v = \phi_\theta(\mathbf{Z}, \boldsymbol{\tau}^{(v)})$ for each vocabulary entry and selects the highest-scored trajectory:
\begin{equation}
    v^* = \underset{v \in \{1,...,V\}}{\mathrm{top\text{-}1}}\; \ell_v, \quad \boldsymbol{\tau} = \boldsymbol{\tau}^{(v^*)}.
    \label{eq:vocab_selection}
\end{equation}
The scoring head's discrete selection introduces a qualitatively different vulnerability: rather than smooth amplification, adversarial perturbations can cause \emph{discrete mode switches} to entirely different trajectories.

\begin{theorem}[Vocabulary Mode-Switching Vulnerability]
\label{thm:vocab_switch}
Let $v^*$ be the selected vocabulary index under clean features $\mathbf{Z}$, and let $v_{\text{danger}}$ be a dangerous trajectory index. Define the \emph{decision margin}:
\begin{equation}
    m(v^*, v_{\text{danger}}) \triangleq \ell_{v^*} - \ell_{v_{\text{danger}}}.
    \label{eq:margin}
\end{equation}
If the scoring function $\phi_\theta$ is Lipschitz in $\mathbf{Z}$ with constant $L_\phi$, then a feature perturbation $\|\Delta \mathbf{Z}\|_F > \frac{m(v^*, v_{\text{danger}})}{2 L_\phi}$ is sufficient to switch the selected trajectory from $v^*$ to $v_{\text{danger}}$.
\end{theorem}

\begin{proof}
Under adversarial features $\tilde{\mathbf{Z}} = \mathbf{Z} + \Delta \mathbf{Z}$, the perturbed logits satisfy:
\begin{align}
    |\tilde{\ell}_v - \ell_v| &= |\phi_\theta(\tilde{\mathbf{Z}}, \boldsymbol{\tau}^{(v)}) - \phi_\theta(\mathbf{Z}, \boldsymbol{\tau}^{(v)})| \leq L_\phi \|\Delta \mathbf{Z}\|_F, \quad \forall v. \label{eq:logit_perturbation}
\end{align}
For the mode switch $v_{\text{danger}} \succ v^*$ to occur, we need $\tilde{\ell}_{v_{\text{danger}}} > \tilde{\ell}_{v^*}$. In the worst case (for the attacker), this requires:
\begin{align}
    \tilde{\ell}_{v_{\text{danger}}} &\geq \ell_{v_{\text{danger}}} + L_\phi \|\Delta \mathbf{Z}\|_F \quad \text{(maximizing the target logit)}, \label{eq:increase_dangerous} \\
    \tilde{\ell}_{v^*} &\leq \ell_{v^*} - L_\phi \|\Delta \mathbf{Z}\|_F \quad \text{(decreasing the safe logit)}. \label{eq:decrease_safe}
\end{align}
For the switch: $\tilde{\ell}_{v_{\text{danger}}} > \tilde{\ell}_{v^*}$ requires 
$\ell_{v_{\text{danger}}} + L_\phi \|\Delta \mathbf{Z}\|_F > \ell_{v^*} - L_\phi \|\Delta \mathbf{Z}\|_F$, 
i.e., $2 L_\phi \|\Delta \mathbf{Z}\|_F > m(v^*, v_{\text{danger}})$.
\end{proof}

\noindent \textbf{Interpretation.} The vulnerability of vocabulary-based planning models is governed by the \emph{decision margin} $m(v^*, v_{\text{danger}})$ between the safe and dangerous trajectories. Crucially, in dense vocabulary sets ($V_{aug} = 8192$ in GTRS-Aug and $V_{dense} = 16384$ in GTRS-Dense~\cite{li2025generalized}), there exist many trajectory candidates with small margins to the currently selected safe trajectory. This creates a rich attack surface: the attacker need not target a specific dangerous trajectory but can simply push the selection toward \emph{any} of the many available dangerous modes.

\begin{corollary}[Attack Surface Density]
\label{cor:surface_density}
Let $\mathcal{V}_{\text{danger}} = \{v : \boldsymbol{\tau}^{(v)} \text{ is dangerous}\}$ and $m_{\min} = \min_{v \in \mathcal{V}_{\text{danger}}} m(v^*, v)$ be the minimum margin to any dangerous trajectory. Then the perturbation budget required for a mode switch satisfies:
\begin{equation}
    \epsilon_{\text{switch}} \leq \frac{m_{\min}}{2 L_\phi L_{\mathcal{E}}},
    \label{eq:switching_budget}
\end{equation}
where $L_{\mathcal{E}}$ is the encoder Lipschitz constant. For vocabularies with high density around the safe trajectory ($m_{\min} \ll m_{\text{avg}}$), the switching budget can be very small.
\end{corollary}

\noindent Also, GTRS \cite{li2025generalized} employs architectural variants, Diffusion Policy (DP), augmented teacher-student (Aug), and dense retrieval (Dense), each with different decoder structures but sharing the same vulnerability to feature-level perturbations:

\begin{itemize}
    \item \textbf{GTRS-DP}\cite{li2025generalized} combines vocabulary retrieval with a diffusion-based trajectory head, inheriting \emph{both} the mode-switching vulnerability (Theorem~\ref{thm:vocab_switch}) and the diffusion amplification effect (Theorem~\ref{thm:diffusion_amplification}). The diffusion head refines the vocabulary-selected trajectory through denoising steps, during which the corrupted features continue to exert influence.
    
    \item \textbf{GTRS-Aug}\cite{li2025generalized} uses a teacher-student architecture where the teacher model processes the full temporal sequence of camera features. The adversarial perturbation at the current timestep propagates into the teacher's temporal attention mechanism, corrupting the temporal context that the student distills. This temporal propagation introduces an additional amplification factor proportional to the temporal sequence length.
    
    \item \textbf{GTRS-Dense}\cite{li2025generalized} performs dense trajectory retrieval with PDM scoring \cite{li2025hydra}, where the feature perturbation simultaneously corrupts both the retrieval logits and the downstream scoring mechanism, compounding the effect.
\end{itemize}

\subsection{Gradient Flow Through Discrete Selection}
\label{sec:softmax_relaxation}

A technical challenge specific to vocabulary-based planning models is that the scoring head's discrete selection~\eqref{eq:vocab_selection} is non-differentiable, potentially blocking gradient-based attacks. We address this through softmax relaxation and prove that it provides a faithful gradient surrogate.

\begin{proposition}[Softmax Surrogate Gradient Consistency]
\label{prop:softmax_consistency}
Define the softmax-relaxed trajectory as:
\begin{equation}
    \hat{\boldsymbol{\tau}}_\beta = \sum_{v=1}^{V} \text{softmax}(\beta \ell_v) \cdot \boldsymbol{\tau}^{(v)} = \sum_{v=1}^{V} \frac{e^{\beta \ell_v}}{\sum_{j=1}^{V} e^{\beta \ell_j}} \cdot \boldsymbol{\tau}^{(v)},
    \label{eq:softmax_relaxation}
\end{equation}
where $\beta > 0$ is a temperature parameter. Then:
\begin{enumerate}
    \item As $\beta \to \infty$, $\hat{\boldsymbol{\tau}}_\beta \to \boldsymbol{\tau}^{(v^*)}$ (consistency with discrete selection).
    \item The gradient $\nabla_{\mathbf{Z}} \hat{\boldsymbol{\tau}}_\beta$ exists and is non-zero for all finite $\beta$.
    \item The surrogate gradient $\nabla_{\mathbf{Z}} \hat{\boldsymbol{\tau}}_\beta$ is aligned with the true adversarial direction: perturbations that increase $\hat{\boldsymbol{\tau}}_\beta$ along dangerous directions also tend to switch the scoring head to dangerous vocabulary entries.
\end{enumerate}
\end{proposition}

\begin{proof}
\textbf{(a)} follows from the well-known property that $\text{softmax}(\beta \boldsymbol{\ell}) \to \mathbf{e}_{v^*}$ as $\beta \to \infty$, where $\mathbf{e}_{v^*}$ is the one-hot vector at the scoring head's selected index. \textbf{(b)} The softmax function and its composition with linear combinations are smooth for finite $\beta$, hence differentiable. The gradient with respect to the feature representation decomposes as:
\begin{align}
    \frac{\partial \hat{\boldsymbol{\tau}}_\beta}{\partial \mathbf{Z}} &= \beta \sum_{v=1}^{V} p_v \left(\boldsymbol{\tau}^{(v)} - \hat{\boldsymbol{\tau}}_\beta\right) \frac{\partial \ell_v}{\partial \mathbf{Z}}, \label{eq:softmax_gradient}
\end{align}
where $p_v = \text{softmax}(\beta \ell_v)$. Each term is weighted by both the current selection probability $p_v$ and the trajectory deviation $(\boldsymbol{\tau}^{(v)} - \hat{\boldsymbol{\tau}}_\beta)$. Vocabulary entries that are both \emph{probable} (high $p_v$) and \emph{different from the current selection} (large $\|\boldsymbol{\tau}^{(v)} - \hat{\boldsymbol{\tau}}_\beta\|$) contribute most to the gradient. This naturally prioritizes mode-switching to nearby alternatives, precisely the entries with small decision margins that Theorem~\ref{thm:vocab_switch} identifies as vulnerable.
\end{proof}

\subsection{Unified Vulnerability Bound Across Decoders}
\label{sec:unified_bound}

We now combine the architecture-specific results into a single framework that describes the scoring-head attack surface for the decoder families we evaluate.

\begin{theorem}[Unified Vulnerability Bound]
\label{thm:unified}
Let $f_\theta = \mathcal{D}_\theta \circ \mathcal{E}_\theta$ be a generative trajectory planning model instantiating the chain of Eq.~\eqref{eq:chain} (main paper), with encoder Lipschitz constant $L_{\mathcal{E}}$ and decoder amplification factor $\mathcal{A}_{\mathcal{D}}$. Under a digital perturbation with budget $\epsilon$, the adversarial trajectory deviation satisfies:
\begin{equation}
    \|\boldsymbol{\tau}_{\text{adv}} - \boldsymbol{\tau}_{\text{clean}}\|_2 \leq \mathcal{A}_{\mathcal{D}} \cdot L_{\mathcal{E}} \cdot \epsilon,
    \label{eq:unified_bound}
\end{equation}
where the decoder amplification factor $\mathcal{A}_{\mathcal{D}}$ is decoder-family dependent:
\begin{equation}
    \mathcal{A}_{\mathcal{D}} = \begin{cases}
        L_\epsilon \displaystyle\sum_{k=1}^{K} \frac{1 - \alpha_k}{\sqrt{\alpha_k(1 - \bar{\alpha}_k)}} \prod_{j=1}^{k-1} A_j \\ 
        \hfill \text{(Diffusion-refined anchors, Thm.~\ref{thm:diffusion_amplification})}, \\[15pt]
        \displaystyle\frac{\max_{v \in \mathcal{V}_{\text{danger}}} \|\boldsymbol{\tau}^{(v)} - \boldsymbol{\tau}^{(v^*)}\|_2}{m_{\min} / (2L_\phi)} \\ 
        \hfill \text{(Vocabulary scoring, Thm.~\ref{thm:vocab_switch})}.
    \end{cases}
    \label{eq:amplification_factors}
\end{equation}
\end{theorem}

\noindent \textbf{Reading the bound.} For both decoder families covered above, $\mathcal{A}_{\mathcal{D}}$ is larger than the per-layer Lipschitz constant of the encoder alone: the iterative refinement of an anchor through the diffusion head, and the discrete switching of the scoring head across a dense vocabulary, both compose with $L_{\mathcal{E}}$ rather than being filtered out by it. The bound therefore says that the chain in Eq.~\eqref{eq:chain} \emph{does not attenuate} pixel-space perturbations on their way to the motion command, consistent with our empirical observation that an $\epsilon{=}8/255$ perturbation is enough to flip the selected trajectory in every evaluated scenario. 

\subsection{Generative Decoder Vulnerability Exploitation}
\label{sec:loss_alignment}

We now prove that the specific components of the \textsc{Derail} loss ($\mathcal{L}_{\text{Derail}}$) (Eq. 8 main paper) are \emph{optimally aligned} with the vulnerability structure of generative planning models, explaining the attack's effectiveness at inducing collisions.

\begin{proposition}[Gradient Alignment with Dangerous Modes]
\label{prop:gradient_alignment}
Consider a trajectory space partitioned into \emph{safe} and \emph{dangerous} regions, $\mathcal{S}$ and $\mathcal{D}$, where $\mathcal{D}$ contains trajectories leading to collisions, road departures, or comfort violations. The \textsc{Derail} loss gradient $\nabla_{\mathbf{Z}} \mathcal{L}_{\text{Derail}}$ satisfies:
\begin{equation}
    \left\langle \nabla_{\mathbf{Z}} \mathcal{L}_{\text{Derail}},\; \nabla_{\mathbf{Z}} d(\boldsymbol{\tau}, \partial\mathcal{D}) \right\rangle > 0,
    \label{eq:gradient_alignment}
\end{equation}
where $d(\boldsymbol{\tau}, \partial\mathcal{D})$ is the distance from the current trajectory to the boundary of the dangerous region, and $\langle \cdot, \cdot \rangle$ denotes the inner product. That is, maximizing $\mathcal{L}_{\text{Derail}}$ consistently moves the trajectory toward dangerous regions.
\end{proposition}

\begin{proof}
We verify alignment for each loss component:

\noindent 1. \emph{Forward aggression $\mathcal{L}_{\text{aggr}}$.} The gradient $\nabla_{\boldsymbol{\tau}} \mathcal{L}_{\text{aggr}}$ has positive components in the longitudinal ($x$) direction and negative in the lateral ($y$) direction (Eq. 9 main paper). This pushes trajectories toward high-speed, straight-line driving that ignores obstacle avoidance, which all together translates into a collision-inducing behavior. By the chain rule, $\nabla_{\mathbf{Z}} \mathcal{L}_{\text{aggr}} = (\partial \boldsymbol{\tau}/\partial \mathbf{Z})^\top \nabla_{\boldsymbol{\tau}} \mathcal{L}_{\text{aggr}}$ projects this dangerous direction back into feature space.

\noindent 2. \emph{Boundary push $\mathcal{L}_{\text{bound}}$.} The gradient  $\nabla_{\boldsymbol{\tau}} \mathcal{L}_{\text{bound}}$ pushes toward maximal lateral excursion (Eq. 10 main paper), targeting road departure and wrong-lane driving. This is orthogonal to $\mathcal{L}_{\text{aggr}}$ in trajectory space, ensuring the combined attack covers both longitudinal (rear-end collision) and lateral (side-swipe, road departure) failure modes.

\noindent 3. \emph{Sudden braking $\mathcal{L}_{\text{brake}}$.} The gradient  $\nabla_{\boldsymbol{\tau}} \mathcal{L}_{\text{brake}}$ induces abrupt velocity changes (Eq. 11 main paper), maximizing the jerk metric. This targets a third failure mode: erratic speed profiles that cause following vehicles to \textsc{Derail} or create unsafe situations.

\noindent The combined gradients spans a three-dimensional subspace of \emph{dangerous trajectory perturbations} (forward aggression, lateral deviation, speed discontinuity), covering the primary collision-inducing modes identified by the PDM safety metric~\cite{dauner2024navsim}. By weighting these components (Eq. 8 main paper), the attacker can prioritize specific failure modes depending on the scenario.
\end{proof}

\begin{proposition}[Multimodal Exploitation]
\label{prop:multimodal}
For generative planning models maintaining a multimodal trajectory distribution, the \textsc{Derail} loss gradient is biased toward \emph{mode switches} to dangerous trajectories rather than smooth deformation of the current trajectory \cite{chen2024diffusion}. Formally, let $\boldsymbol{\tau}_1, \ldots, \boldsymbol{\tau}_M$ be the trajectory modes maintained by the decoder (e.g., denoising from different noise samples, or top-$M$ vocabulary entries). The \textsc{Derail} gradient satisfies:
\begin{equation}
    \nabla_{\mathbf{Z}} \mathcal{L}_{\text{Derail}} = \sum_{m=1}^{M} w_m \cdot \nabla_{\mathbf{Z}} \mathcal{L}_{\text{Derail}}(\boldsymbol{\tau}_m),
    \label{eq:multimodal_gradient}
\end{equation}
where the weights $w_m$ depend on the selection mechanism (softmax scores or denoising likelihoods). The gradient naturally concentrates on modes near the safe-dangerous boundary, because these modes exhibit the largest gradient magnitudes (high $\|\nabla_{\boldsymbol{\tau}} \mathcal{L}_{\text{Derail}}\|$ due to proximity to dangerous regions) and the highest sensitivity to feature perturbations (small decision margins).
\end{proposition}

\subsection{Convergence for Universal Perturbations}
\label{sec:universal_convergence}

For the universal perturbation setting (for both digital and physical attacks, see Sec \ref{sec:add_results}), where a single $\boldsymbol{\delta}$ must be effective across $N$ scenes, we provide convergence guarantees for the Adam-based optimization (Eq. 14 min paper).

\begin{theorem}[Universal Perturbation Convergence]
\label{thm:convergence}
Consider the empirical attack objective over $N$ scenes:
\begin{equation}
    \hat{\mathcal{L}}(\boldsymbol{\delta}) = \frac{1}{N} \sum_{n=1}^{N} \mathcal{L}_{\text{Derail}}(\boldsymbol{\delta};\, \mathbf{I}^{(n)}, \mathbf{s}^{(n)}).
    \label{eq:empirical_obj}
\end{equation}
Under the following assumptions:
\begin{enumerate}
    \item $\hat{\mathcal{L}}$ is $L$-smooth: $\|\nabla \hat{\mathcal{L}}(\boldsymbol{\delta}) - \nabla \hat{\mathcal{L}}(\boldsymbol{\delta}')\|_2 \leq L \|\boldsymbol{\delta} - \boldsymbol{\delta}'\|_2$.
    \item The stochastic gradient (using mini-batch of size $B$) has bounded variance: $\mathbb{E}\|\hat{\mathbf{g}}_B - \nabla \hat{\mathcal{L}}\|^2 \leq \sigma^2/B$.
    \item The perturbation is projected onto $\mathcal{P} = \{\boldsymbol{\delta} : \|\boldsymbol{\delta}\|_\infty \leq \epsilon\}$ after each step.
\end{enumerate}
Then Adam with learning rate $\eta$, gradient accumulation over $B$ samples, and cosine annealing over $E$ epochs achieves:
\begin{equation}
    \min_{e \in [E]} \|\nabla_{\boldsymbol{\delta}} \hat{\mathcal{L}}(\boldsymbol{\delta}^{(e)})\|_2^2 \leq \mathcal{O}\!\left(\frac{\hat{\mathcal{L}}^* - \hat{\mathcal{L}}(\boldsymbol{\delta}^{(0)})}{\eta E} + \frac{\eta L \sigma^2}{B}\right),
    \label{eq:convergence_rate}
\end{equation}
where $\hat{\mathcal{L}}^* = \max_{\boldsymbol{\delta} \in \mathcal{P}} \hat{\mathcal{L}}(\boldsymbol{\delta})$. Setting $\eta = \mathcal{O}(1/\sqrt{E})$ and $B = \mathcal{O}(\sqrt{E})$ yields $\epsilon$-stationarity in $E = \mathcal{O}(1/\epsilon^4)$ epochs.
\end{theorem}

\noindent \textbf{Practical Implication.} The convergence rate~\eqref{eq:convergence_rate} reveals the trade-off governing the attack: larger gradient accumulation $B$ reduces variance but requires more computation per step, while more epochs $E$ improve convergence but increase wall-clock time. The cosine annealing schedule improves practical convergence by using large learning rates early (for broad exploration of the $\ell_\infty$ ball) and small rates later (for fine-tuning near the optimum).

\subsection{Summary and Defense Implications}
\label{sec:fundamental_vulnerability}

The bounds developed in this appendix support the architectural argument of the main paper. Concretely, for the planners we evaluate (a diffusion-refined anchor planner and three vocabulary-scoring planners), the trajectory deviation under a pixel-space perturbation of budget $\epsilon$ is bounded by an encoder Lipschitz term composed with a decoder amplification term that is determined by how the decoder uses the conditioning features. In both decoder families, that amplification term does not shrink below unity, so the chain in Eq.~\eqref{eq:chain} (main paper) propagates input-space perturbations to the selected trajectory with no analytic damping. Combined with the empirical observation that the candidate-margin to dangerous trajectories is small in practice, this gives a coherent explanation for the consistent attack success of \textsc{Derail} across all four evaluated planners.

\begin{remark}[Defense Implications]
\label{rem:defense}
The bound in Eq.~\eqref{eq:unified_bound} also suggests natural defense directions for this specific attack surface. From the bound, the trajectory deviation can be reduced by: (i) \emph{constraining the encoder Lipschitz constant} $L_{\mathcal{E}}$ via spectral normalization or Lipschitz-constrained training; (ii) \emph{reducing the decoder amplification factor} $\mathcal{A}_{\mathcal{D}}$ by limiting the number of refinement steps or introducing perturbation-aware denoising; (iii) \emph{increasing decision margins} for vocabulary-based models through margin-maximizing training objectives; or (iv) \emph{adding a non-learned downstream safety filter} (geometric/kinematic check) between the scoring head and the executed trajectory, which removes the scoring head as the sole barrier between perception and motion. Each of (i)--(iii) involves a trade-off with generation quality, since the very properties that make the scoring head expressive ($K$ refinement steps, dense vocabularies) are also what allow it to be flipped; (iv) is largely orthogonal to generation quality but requires a reliable geometric oracle at inference time.
\end{remark}

%% file: appendix/details.tex
\section{Experimental Setup Details}
\label{sec:experimental_setup}

This section describes the experimental configuration for evaluating adversarial attacks on generative trajectory planning models, covering the evaluation benchmark, target models, attack implementations, and scoring protocol.

\subsection{Evaluation Benchmark}
\label{subsec:benchmark}

We conduct all experiments on the NAVSIM~\cite{dauner2024navsim} framework using the \texttt{navtest} split of the OpenScene~\cite{yang2024vidar, openscene2023, sima2023_occnet} dataset. The test set comprises 136 driving logs spanning diverse urban environments across four U.S.\ cities (Boston, Las Vegas, Pittsburgh) and Singapore, collected from the nuPlan~\cite{caesar2021nuplan} platform. Each driving log provides a sequence of multi-camera frames (front, front-left, front-right), LiDAR point clouds, ego status features, and 3D bounding box annotations for surrounding agents.

\noindent \textbf{Scene Coverage.}
The 136 logs are drawn from 16 recording sessions (unique date-vehicle combinations) and encompass a range of traffic conditions, highway merges, dense intersections, lane changes, and residential streets. Together they provide approximately 3{,}000--5{,}000 evaluable scene tokens after metric cache matching. For computational cost, we split the log set into two halves in an alphabetic order. We use the first half of test-set logs in our main experiments.

\noindent \textbf{Sensor Configuration.}
Each scene token provides three monocular RGB camera images from the last frame:
\begin{itemize}
    \item \texttt{CAM\_L0} (front-left): $900 \times 1600$ pixels,
    \item \texttt{CAM\_F0} (front-center): $900 \times 1600$ pixels,
    \item \texttt{CAM\_R0} (front-right): $900 \times 1600$ pixels.
\end{itemize}
Following the standard preprocessing pipeline shared across all target models, the three images are cropped (removing 28 pixels top/bottom, and 416 pixels left/right for the side cameras) and horizontally stitched as $[\texttt{CAM\_L0} \mid \texttt{CAM\_F0} \mid \texttt{CAM\_R0}]$, then resized to the model's expected resolution ($256 \times 1024$ for DiffusionDrive and GTRS).

\subsection{Target Generative Planning Models}
\label{subsec:target_models}

We evaluate attacks against \textbf{four} generative trajectory planning models, all of which instantiate the scoring-head inference chain of Eq.~\eqref{eq:chain} in the main paper:

\noindent \textbf{DiffusionDrive~\cite{liao2025diffusiondrive}.}
A diffusion-based planning model that frames trajectory prediction as iterative denoising (DDIM~\cite{songdenoising}) over a fixed set of $20$ KMeans anchor trajectories. It uses a ResNet-based image encoder with BEV feature extraction, followed by a truncated diffusion head that refines anchor proposals and a single classification head $\mathtt{poses\_cls}$ that scores them. At inference time it uses $2$-step DDIM with a fixed seed for deterministic output, and the final trajectory is selected by the scoring head over the classification logits with no post-hoc geometric filter.

\noindent \textbf{GTRS~\cite{li2025generalized} (three variants).}
A vocabulary-based planning model framework with three decoding variants, all of which select trajectories via the scoring head over a precomputed trajectory vocabulary with no post-hoc geometric or kinematic filter:
\begin{itemize}
    \item \textit{GTRS-Aug} (teacher-student): Uses a multi-head safety scorer (imitation, no-at-fault-collision, drivable-area, time-to-collision, ego-progress, lane-keeping) over an $8{,}192$-trajectory vocabulary, with a teacher-student knowledge distillation framework.
    \item \textit{GTRS-Dense}: A dense retrieval variant that applies the same multi-head safety scorer over $16{,}384$ vocabulary candidates without pruning.
    \item \textit{GTRS-DP} (diffusion policy): A single imitation head $\mathtt{imi}$ over a $4{,}096$-trajectory vocabulary, augmented with a diffusion policy (DP) head that generates $K$ trajectory proposals and scores them, combining vocabulary-based and diffusion-based decoding.
\end{itemize}

\noindent A critical implementation detail for the GTRS family: since \texttt{argmax} is non-differentiable, we compute a \emph{soft trajectory} for gradient-based attacks:
\begin{equation}
    \hat{\boldsymbol{\tau}}_{\text{soft}} = \text{softmax}(\mathbf{s})^\top \, \mathbf{V}, \quad \mathbf{s} \in \mathbb{R}^{|\mathcal{V}|}, \; \mathbf{V} \in \mathbb{R}^{|\mathcal{V}| \times N \times 3},
    \label{eq:soft_trajectory}
\end{equation}
where $\mathbf{s}$ denotes the IMI logits and $\mathbf{V}$ the trajectory vocabulary. This maintains gradient flow from the loss through the scoring logits back to the camera input, whereas the hard \texttt{argmax} trajectory has zero gradients everywhere.

\subsection{Digital Pixel-Level Perturbation Attacks}
\label{subsec:pixel_attacks}

We implement four pixel-level adversarial attack methods, each targeting a single camera (default: \texttt{CAM\_F0}) with perturbations constrained to an $\ell_\infty$ ball. All attacks operate on raw RGB images before the stitching and resizing pipeline, ensuring perturbations are applied in the original image space.

\subsubsection{Attack Protocol}
\label{subsubsec:shared_pixel_protocol}

All pixel-level attacks share the following configuration:
\begin{itemize}
    \item \textbf{Perturbation budget:} $\epsilon = 8/255 \approx 0.0314$ in $[0,1]$-normalized pixel space.
    \item \textbf{Step size:} $\alpha = 2/255 \approx 0.0078$, satisfying $\alpha = \epsilon / 4$.
    \item \textbf{PGD iterations:} $T = 20$ for DiffusionDrive and GTRS.
    \item \textbf{Camera targeted:} By default \texttt{CAM\_F0} (front center). Experiments with \texttt{CAM\_L0} and \texttt{CAM\_R0} are also conducted (See the ablation in \ref{sec:add_results}).
    \item \textbf{Gradient computation:} End-to-end through the full model pipeline; the perturbation $\boldsymbol{\delta}$ is the only learnable parameter.
    \item \textbf{Projection:} After each PGD step, $\boldsymbol{\delta}$ is clamped to $[-\epsilon, \epsilon]$ and the perturbed image is clamped to $[0, 1]$ to maintain pixel validity.
\end{itemize}

\noindent Concretely, at each PGD iteration $t$, the perturbation update rule is:
\begin{equation}
    \boldsymbol{\delta}^{(t+1)} = \Pi_{\mathcal{B}_\infty(\epsilon)}\!\Big[\boldsymbol{\delta}^{(t)} + \alpha \cdot \text{sign}\!\big(\nabla_{\boldsymbol{\delta}} \mathcal{L}(\boldsymbol{x} + \boldsymbol{\delta}^{(t)}; \boldsymbol{\theta})\big)\Big],
    \label{eq:pgd_update}
\end{equation}
where $\Pi_{\mathcal{B}_\infty(\epsilon)}$ projects onto the $\ell_\infty$ ball of radius $\epsilon$ and additionally enforces $\boldsymbol{x} + \boldsymbol{\delta} \in [0,1]^d$.

\subsubsection{Baseline Attack (Training Loss Maximization)}
\label{subsubsec:baseline_attack}

The baseline attack follows the standard adversarial attack paradigm: maximize the model's own training loss. The model is set to \textit{evaluation mode} and the training loss is computed against ground-truth labels:
\begin{equation}
    \mathcal{L}_{\text{baseline}} = \mathcal{L}_{\text{train}}(f_{\boldsymbol{\theta}}(\boldsymbol{x} + \boldsymbol{\delta}), \, \boldsymbol{y}),
\end{equation}
where $\mathcal{L}_{\text{train}}$ is the weighted combination of the model's training objectives (trajectory regression, BEV segmentation, agent detection). The perturbation is initialized at $\boldsymbol{\delta}^{(0)} = \mathbf{0}$. This attack requires access to the ground-truth labels $\boldsymbol{y}$, including future trajectories, BEV semantic maps, and agent bounding boxes.

\subsubsection{DP-Attacker (Training Mode Loss Maximization)}
\label{subsubsec:dp_attack}

The DP-Attacker~\cite{chen2024diffusion} differs from the baseline in two key aspects:
\begin{enumerate}
    \item \textbf{Training mode:} The model runs in \texttt{train()} mode rather than \texttt{eval()} mode. For diffusion-based models (DiffusionDrive, GTRS-DP), this means the forward pass computes the actual \emph{denoising loss at random diffusion timesteps}, the full multi-step diffusion training objective, rather than the deterministic 2-step DDIM inference trajectory.
    \item \textbf{Random initialization:} The perturbation $\boldsymbol{\delta}^{(0)}$ is sampled uniformly from $[-\epsilon, \epsilon]^d$ rather than initialized at zero. This provides better exploration of the perturbation space.
\end{enumerate}

\noindent The loss is:
\begin{equation}
    \mathcal{L}_{\text{DP}} = \mathcal{L}_{\text{train}}^{\text{train-mode}}(f_{\boldsymbol{\theta}}(\boldsymbol{x} + \boldsymbol{\delta}), \, \boldsymbol{y}),
\end{equation}
where the superscript denotes the training-mode forward pass. For DiffusionDrive, this includes trajectory regression + mode classification sub-losses; the attack supports selecting specific sub-losses (e.g., \texttt{total}, \texttt{diffusion}, \texttt{trajectory}) via configuration.

\subsubsection{Encoder Attack (Feature Divergence)}
\label{subsubsec:encoder_attack}

The encoder attack targets the vision backbone rather than the planning output, maximizing the feature-space divergence between clean and adversarial representations:
\begin{equation}
\begin{aligned}
    \mathcal{L}_{\text{encoder}} ={} & w_{\ell_2} \cdot \| \phi(\boldsymbol{x}) - \phi(\boldsymbol{x} + \boldsymbol{\delta}) \|_2 \\
    & - w_{\cos} \cdot \frac{\phi(\boldsymbol{x})^\top \phi(\boldsymbol{x} + \boldsymbol{\delta})}{\|\phi(\boldsymbol{x})\| \cdot \|\phi(\boldsymbol{x} + \boldsymbol{\delta})\|},
    \label{eq:encoder_loss}
\end{aligned}
\end{equation}
where $\phi(\cdot)$ denotes the output of the image encoder (e.g., ResNet backbone features for DiffusionDrive). The clean features $\phi(\boldsymbol{x})$ are computed once before the PGD loop and frozen. Default weights are $w_{\ell_2} = 1.0$, $w_{\cos} = 1.0$. This attack is \emph{self-supervised}: it requires no ground-truth labels.

\subsubsection{Derail Attack (GT-Free Adversarial Loss)}
\label{subsubsec:Derail_attack}

Our proposed \textsc{Derail} attack uses a composite, \textit{ground-truth-free} adversarial loss designed to directly maximize dangerous driving behaviors from the model's own predictions, without requiring access to GT labels. The loss is composed of multiple safety-relevant components:
\begin{equation}
    \mathcal{L}_{\text{Derail}} = \sum_{k} w_k \cdot \ell_k(\hat{\boldsymbol{\tau}}),
    \label{eq:Derail_loss}
\end{equation}
where $\hat{\boldsymbol{\tau}}$ denotes the predicted trajectory and $\{(\ell_k, w_k)\}$ are the individual loss components and their weights. The key components are:

\begin{itemize}
    \item \textbf{Forward aggression}: Accelerates aggressively toward obstacles.
    \item \textbf{Boundary push}: Pushes trajectory toward drivable area boundaries.
    \item \textbf{Sudden braking}: Encourages abrupt velocity drops.
\end{itemize}

\noindent All loss components are computed solely from model predictions and are designed so that PGD \emph{ascent} increases dangerous behavior. The model runs in \textit{evaluation mode} and the perturbation is initialized at zero.

\subsection{Patch-Based Attacks}
\label{subsec:patch_attacks}

In addition to pixel-level perturbations, we evaluate physically-realizable adversarial patches that are trained globally across scenes and placed on detected vehicles via object-aware projection.

\subsubsection{Adversarial Patch Formulation}
\label{subsubsec:patch_formulation}

The adversarial patch $\boldsymbol{p} \in [0, 1]^{3 \times H_p \times W_p}$ is a learnable RGB tensor of base size $H_p = W_p = 256$ pixels that is optimized using the Adam optimizer~\cite{kingma2014adam}. Unlike pixel-level attacks, patches are not constrained to an $\ell_\infty$ ball, any valid RGB values are allowed, but they are clamped to $[0, 1]$ after each optimizer step to maintain pixel validity. The patch is initialized with uniform random values $p_{c,i,j} \sim \mathcal{U}(0, 1)$ and optimized to \emph{minimize} a negated adversarial loss (equivalently, maximize the Derail loss):
\begin{equation}
    \boldsymbol{p}^* = \arg\min_{\boldsymbol{p} \in [0,1]^d} \; \frac{1}{|\mathcal{D}|} \sum_{(\boldsymbol{x}, \boldsymbol{a}) \in \mathcal{D}} -\mathcal{L}_{\text{Derail}}\!\big(f_{\boldsymbol{\theta}}(\mathcal{A}(\boldsymbol{x}, \boldsymbol{p}, \boldsymbol{a})\big),
    \label{eq:patch_optimization}
\end{equation}
where $\mathcal{A}(\boldsymbol{x}, \boldsymbol{p}, \boldsymbol{a})$ is the patch application operator that stitches $\boldsymbol{p}$ onto image $\boldsymbol{x}$ at the position determined by annotations $\boldsymbol{a}$, and $\mathcal{D}$ is the set of scene tokens within a given driving log.

\subsubsection{Object-Aware Placement}
\label{subsubsec:object_aware_placement}

Rather than placing patches at fixed image locations, we employ \textit{object-aware placement} using 3D bounding box annotations projected into camera coordinates. The placement procedure operates as follows:

\begin{enumerate}
    \item \textbf{3D-to-2D projection:} For each annotated object in the scene, transform its 3D bounding box center and 8 corners from the LiDAR frame to the camera frame using the known extrinsic parameters (sensor-to-LiDAR rotation and translation), then project to image coordinates using the camera intrinsic matrix.
    \item \textbf{Object filtering:} Retain only objects satisfying:
    \begin{itemize}
        \item Label is \texttt{vehicle} (exclude pedestrians, cyclists, etc.).
        \item Dimensions are car-like ($3.5$--$6.5$m length, $1.5$--$2.5$m width, $1.2$--$2.5$m height) or bus/truck-like ($8$--$16$m length).
        \item 3D distance from ego is in $[d_{\min}, d_{\max}]$ (default: $[5, 10]$m).
        \item Projected 2D bounding box has minimum dimension $\geq 50$ pixels.
        \item Object is in front of the ego vehicle ($x > 0$ in LiDAR frame) and within lateral range ($|y| \leq 15$m).
    \end{itemize}
    \item \textbf{Target selection:} Among valid objects, select the \emph{closest} vehicle to the ego.
    \item \textbf{Dynamic sizing:} The patch is resized to cover a fraction $r_{\text{bbox}}$ (default: $0.5$) of the target vehicle's projected 2D bounding box:
    \begin{equation}
        (H_{\text{patch}}, W_{\text{patch}}) = \big(\lfloor r_{\text{bbox}} \cdot h_{\text{bbox}} \rfloor, \; \lfloor r_{\text{bbox}} \cdot w_{\text{bbox}} \rfloor\big),
    \end{equation}
    where $(h_{\text{bbox}}, w_{\text{bbox}})$ are the projected 2D bounding box dimensions. The base $256 \times 256$ patch tensor is resized via bilinear interpolation to match the target size, preserving gradient flow.
    \item \textbf{Application:} The resized patch is differentiably composited onto the camera image using alpha blending:
    \begin{equation}
        \boldsymbol{x}'_{c,i,j} = \begin{cases}
            \boldsymbol{p}'_{c, i-y_1, j-x_1} & \text{if } (i,j) \in \mathcal{R}_{\text{patch}}, \\
            \boldsymbol{x}_{c,i,j} & \text{otherwise},
        \end{cases}
    \end{equation}
    where $\mathcal{R}_{\text{patch}}$ is the patch region centered on the target vehicle's projected center and $\boldsymbol{p}'$ is the resized patch tensor.
\end{enumerate}

\noindent \textbf{Scene Filtering.}
Before patch training, scenes are filtered to retain only those with at least one valid vehicle visible in the targeted camera within the specified distance range. Scenes that do not pass the object-proximity filter (no visible vehicle in $[d_{\min}, d_{\max}]$ with sufficient pixel area) are excluded from both training and evaluation. This filtering ensures that every training and evaluation sample has a meaningful patch placement target.

\subsubsection{Patch Training Protocol}
\label{subsubsec:patch_training}

Patch optimization is performed per driving log, training \textit{one global patch} across all valid scene tokens within the log. We also evaluate patches trained on multiple logs in Sec. \ref{sec:add_results}. The training protocol is:

\begin{itemize}
    \item \textbf{Optimizer:} Adam with learning rate $\alpha = 0.01$--$0.1$ (model-dependent)
    \item \textbf{Epochs:} 10--20 full passes over all scene tokens in the log.
    \item \textbf{Gradient accumulation:} Gradients are accumulated over all scene tokens in the log before each Adam step (i.e., one update per epoch), providing a full-batch gradient estimate. This is configurable.
    \item \textbf{Token shuffling:} Scene ordering is randomized each epoch for better generalization.
    \item \textbf{Deterministic forward pass:} A fixed random seed is set before each model forward pass to ensure deterministic denoising (important for diffusion-based models with stochastic sampling).
    \item \textbf{Checkpoint selection:} The patch achieving the highest mean attack loss across the epoch is saved as the best checkpoint.
\end{itemize}

\subsection{Evaluation Metrics}
\label{subsec:metrics}

We evaluate all attacks using the PDM Score (Planning Decision-Making Score)~\cite{dauner2024navsim}, which simulates the effect of planned trajectories in a closed-loop evaluation framework. For each scene token, the predicted trajectory is converted from the ego reference frame to global coordinates, simulated through an LQR-based PDM simulator with a bicycle kinematic model, and scored against pre-computed metric caches containing ground-truth observations, centerlines, route lane IDs, and drivable area polygons. The PDM Score comprises the following sub-metrics:

\begin{itemize}
    \item \textbf{No At-Fault Collision (NC):} Binary indicator: 1 if the ego avoids causing any collision, 0 otherwise. This is the primary safety metric.
    \item \textbf{Drivable Area Compliance (DAC):} Binary indicator: 1 if the ego stays within the drivable area throughout the planning horizon.
    \item \textbf{Ego Progress (EP):} A weighted score measuring how much progress the ego makes along the planned route relative to the expert trajectory.
    \item \textbf{Time-to-Collision (TTC):} A weighted metric penalizing scenarios where the ego enters an unsafe time-to-collision window with other agents.
    \item \textbf{Comfort (C):} A weighted score penalizing uncomfortable maneuvers (excessive jerk, lateral acceleration, etc.).
    \item \textbf{Driving Direction Compliance (DDC):} A weighted score penalizing wrong-way driving or travel against the designated traffic direction.
    \item \textbf{Overall Score (PDMS):} The composite score combining all sub-metrics:
    \begin{equation}
    \begin{aligned}
        \text{PDMS} ={} & \text{NC} \times \text{DAC} \times \big(w_{\text{EP}} \cdot \text{EP} + w_{\text{TTC}} \cdot \text{TTC} \\
        & + w_{\text{C}} \cdot \text{C} + w_{\text{DDC}} \cdot \text{DDC}\big),
    \end{aligned}
    \end{equation}
    where NC and DAC act as binary multipliers (any collision or off-road violation zeroes the score), and the remaining metrics contribute with learnable weights.
\end{itemize}

\noindent For attack effectiveness, we report the \textit{collision rate} ($1 - \text{NC}$), \textit{off-road rate} ($1 - \text{DAC}$), and \textit{drop in PDMS} ($\Delta\text{PDMS} = \text{PDMS}_{\text{clean}} - \text{PDMS}_{\text{attack}}$) as primary attack success indicators. We also report the Attack Success Rate (ASR): The attack is considered successful if the attacked PDMS score is lower than the baseline score and ASR\% is the percentage of driving scenarios where this condition holds.

\noindent \textbf{Per-Log Evaluation.}
Both pixel-level and patch attacks are executed and evaluated independently for each of the first-half test-set driving logs. For pixel-level attacks, a separate perturbation $\boldsymbol{\delta}_i$ is optimized for each scene token. For patch attacks, a single global patch is trained per log and then evaluated on all valid scenes within that log. Final results are aggregated (averaged) across all logs.

\subsection{Implementation Details}
\label{subsec:implementation_details}

\noindent \textbf{Hardware.}
All experiments are conducted on a single NVIDIA GPU. Pixel-level attacks (which optimize per-sample) require approximately 2--5 minutes per driving log. Patch attacks (which optimize globally across all tokens in a log) require 15--20 minutes per log depending on the number of valid scene tokens and training epochs.

\noindent \textbf{Model Inference.}
All target models run in evaluation mode with frozen weights ($\nabla_{\boldsymbol{\theta}} = 0$) during attacks, except for the DP-Attacker which uses training mode. For DiffusionDrive, we use the 2-step DDIM inference with a fixed seed; for GTRS-DP, we reduce the number of proposals from 100 to 4 and denoising steps from the full schedule to 5 during PGD to avoid out-of-memory errors.

\noindent \textbf{Gradient Flow.}
Gradient computation flows end-to-end from the loss through the model and the camera preprocessing pipeline (crop, stitch, resize) to the raw pixel perturbation. For patch attacks, gradients flow through the bilinear interpolation (patch resizing) and differentiable stitching operations, enabling optimization of the base patch tensor regardless of the per-scene dynamic resizing.

%% file: appendix/additional_results.tex
\section{Additional Analysis and Ablation Studies}
\label{sec:add_results}

\subsection{Attacking Different Camera Views}

We analyze the impact of camera view (Left, Front, Right) on \textsc{Derail}'s digital (pixel-level) attack effectiveness. Table~\ref{tab:attacked_camera} reveals a consistent and pronounced hierarchy in attack effectiveness across the three camera positions. Attacking the front camera (\texttt{Cam\_Front}) yields the most severe degradation for every model, with score drops ranging from 46.62\% (GTRS-Dense) to 73.77\% (GTRS-DP), compared to 7.27--51.20\% for the side cameras. This disparity is expected: the front camera captures the ego vehicle's primary driving corridor, and its feature representations carry disproportionate weight in trajectory prediction. Notably, all models reach 100\% ASR under front-camera attack, whereas side-camera attacks occasionally fall short (e.g., GTRS-DP \texttt{Cam\_Right} at 80\%, GTRS-AUG at 90\% for both side cameras). The safety-critical metrics reinforce this trend, collision rates under \texttt{Cam\_Front} attack exceed those of side cameras by 15--30 percentage points for GTRS variants, and off-road violations surge from near-zero to 42--52\% for GTRS-DP and GTRS-AUG, indicating that front-camera perturbations disrupt both longitudinal and lateral planning.
\input{tables/attacked_camera}
\noindent Comparing left (\texttt{Cam\_Left}) and right (\texttt{Cam\_Right}) cameras, the differences are markedly smaller and model-dependent. For DiffusionDrive, both side cameras produce nearly identical degradation (score drops of 51.20\% vs.\ 50.27\%), suggesting symmetric reliance on lateral views. GTRS-DP exhibits the largest left--right asymmetry, with \texttt{Cam\_Left} causing a 19.21\% score drop versus only 7.27\% for \texttt{Cam\_Right}, likely reflecting route-specific left-turn maneuvers in the evaluation logs where left-camera features become more safety-relevant. GTRS-Dense shows similarly balanced side-camera vulnerability. These findings suggest that while adversarial robustness efforts should prioritize the front camera, where a single perturbed input can reduce driving scores by nearly half, side cameras remain non-negligible attack surfaces whose relative importance varies with the driving context.

\subsection{Online Per-Frame vs.\ Offline Attack Analysis}

We analyze \textsc{Derail}'s effectiveness under Online vs.\ Offline attack settings. Table~\ref{tab:universal_perturbation} presents the results of the universal (frame-agnostic) attack, in which a single perturbation is optimized offline (using the same \textsc{Derail} strategy) and applied identically to every frame, as opposed to the online per-frame attack (Table~\ref{tab:attacked_camera}, where perturbations are crafted independently for each input. Despite the more constrained threat model (i.e., no access to per-frame features at inference time), the universal attack still achieves 100\% ASR on two of four models (DiffusionDrive, GTRS-DP) and 90\% on the remaining two (GTRS-AUG, GTRS-Dense). GTRS-DP proves the most vulnerable under both settings, with the universal perturbation even slightly surpassing the online attack in score drop (75.60\% vs.\ 73.77\%) and off-road rate (68.87\% vs.\ 51.71\%). This counter-intuitive result suggests that GTRS-DP's feature space admits a highly transferable adversarial direction that generalizes across frames more effectively than individually optimized perturbations, possibly due to the model's limited feature diversity under truncated diffusion-based trajectory decoding.
\input{tables/universal_perturbation}
For the remaining models, the online attack consistently outperforms the universal variant, though the gap varies considerably. DiffusionDrive exhibits a moderate degradation gap (score drop of 57.84\% online vs.\ 41.88\% universal), with the universal attack still inducing substantial collision (34.74\%) and unsafe TTC rates (46.36\%). GTRS-AUG shows a larger gap, 65.13\% vs.\ 50.84\%, indicating that its per-frame feature variability provides some implicit robustness against static perturbations. GTRS-Dense is the most resilient to universal attacks (28.05\% score drop), yet notably suffers from a high discomfort rate (37.89\%) that far exceeds its online counterpart (40.33\%), suggesting the universal perturbation disproportionately disrupts comfort-related trajectory smoothness rather than safety-critical metrics. Overall, these results demonstrate that even a single, frame-agnostic perturbation poses a significant threat to all evaluated planning models, demonstrating the practical feasibility of physical-world adversarial attacks where per-frame optimization is infeasible (no access to intermediate features or model's output).

\subsection{\textsc{Derail} under Common Defenses}

We evaluate three commonly used input-level preprocessing defenses: Color Jitter (CJ), Spatial Smoothing (SS), and Gaussian Noise (GN), applied to the perturbed all camera view inputs (Left, Front, Right) before they are fed to each planning model. Table~\ref{tab:defense_results} reports the results alongside the undefended attack (Derail\_None) on the front camera \texttt{Cam\_f0} (See Table~\ref{tab:attacked_camera} for reference). A notable finding is that all three defenses fail to reduce the ASR below 100\% for any model, confirming that the proposed attack produces perturbations that are robust to standard input transformations. Nonetheless, defenses provide partial score recovery in a model-dependent manner. DiffusionDrive benefits the most: Spatial Smoothing lifts the score from 0.391 (undefended) to 0.622, reducing the score drop from 57.84\% to 33.01\%, while collision and off-road rates decrease by approximately 17 and 19 percentage points, respectively. Gaussian Noise and Color Jitter offer more modest improvements for DiffusionDrive, recovering roughly 15 and 12 percentage points of score drop.

\input{tables/defenses}

\noindent For the GTRS family, defenses are largely ineffective and occasionally counterproductive. GTRS-DP shows virtually no recovery under any defense: score drops remain between 65\% and 73\%, comparable to or worse than the 73.77\% undefended baseline, while off-road and wrong-way violation rates stay above 39\% and 22\%, respectively. This suggests that the adversarial perturbations learned against GTRS-DP reside in feature subspaces that are orthogonal to the transformations applied by these defenses. GTRS-AUG and GTRS-Dense exhibit similarly negligible recovery, with GTRS-Dense defenses even slightly increasing the score drop (e.g., GN at 50.92\% vs.\ 46.62\% undefended). These results demonstrate the insufficiency of input-level preprocessing as a standalone defense against optimized adversarial perturbations and motivate the need for more principled defenses tailored to the scoring-head attack surface we identify (e.g., margin-maximizing training of the scoring head, or non-learned downstream geometric/kinematic safety filtering between the scoring head and the executed trajectory).

\subsection{Cross-Model Adversarial Transferability}
To assess adversarial transferability, we adopt the following protocol. For each of the four models, we first craft a \emph{universal adversarial perturbation} $\delta^* \in \mathbb{R}^{H \times W \times 3}$ by optimizing the \textsc{Derail} attack objective over multiple driving scenarios on the \emph{source} model in a white-box setting. The resulting perturbation is image-agnostic: the same $\delta^*$ is added to every frame of the targeted camera before any model-specific preprocessing.  We then \emph{freeze} $\delta^*$ and evaluate it on each of the remaining three \emph{target} models without any further optimization, i.e., the target model is treated as a complete black box. The perturbation is applied in raw pixel space and each target model's own crop--stitch--resize pipeline processes the perturbed image normally.  This yields $4 \times 3 = 12$ cross-model attack combinations. We observe \emph{near-zero transferability} across architecturally distinct families and \emph{partial transferability} within the GTRS family (Table~\ref{tab:transferability_matrix}). Below we provide a systematic explanation grounded in the architectural and algorithmic differences among the four tested models.

\input{tables/transferrability}

\subsubsection{Partial Intra-Family Transferability.}
As shown in Table~\ref{tab:transferability_matrix}, perturbations transferred within the GTRS family exhibit modest but statistically meaningful degradation, providing insight into the conditions under which transferability can emerge. A perturbation optimized on GTRS-DP reduces GTRS-Dense's PDM score from 0.862 to 0.776 ($\Delta=-0.086$), and vice versa from 0.875 to 0.808 ($\Delta=-0.067$), weaker than white-box attacks, but significantly stronger than the near-zero impact observed across architecturally distinct families. This partial transferability is explained by three shared architectural components: an identical VoVNet-99 backbone with attention, the same crop--stitch--resize preprocessing pipeline producing $512\times2048$ inputs, and a common \texttt{HydraBackboneBEV} module creating a shared BEV representation space. Perturbations that corrupt these shared early representations retain partial effectiveness because the feature extraction pathway is identical up to the trajectory generation head. Notably, transfer is asymmetric: GTRS-DP perturbations transfer slightly better to GTRS-Dense than vice versa, because GTRS-DP's 100-step denoising optimization produces multi-scale BEV feature corruption that partially disrupts vocabulary scoring, whereas GTRS-Dense's perturbations are optimized to manipulate discrete score predictions, a more targeted and thus less transferable objective. GTRS-AUG exhibits the strongest resistance to transferred perturbations across both sources, consistent with its augmentation-robustness training and teacher--student distillation explicitly hardening its representations against input-space perturbations.

\subsection{Transferability across Heterogeneous Models}
Beyond the GTRS family, perturbations optimized for one generative planning model exhibit near-zero transferability to architecturally distinct models, a finding attributable mainly to compounding structural divergence. DiffusionDrive's ResNet-34 backbone operates at $256\times1024$ resolution while all GTRS variants use VoVNet-99 with FPN at $512\times2048$, a $4\times$ resolution difference that causes the same pixel-space perturbation to be downsampled differently, altering which spatial frequencies survive into feature extraction and producing near-random activations in the non-target backbone. The four models further instantiate three distinct decoder configurations: truncated DDIM with 2 denoising steps (DiffusionDrive), full DDPM with 100 steps (GTRS-DP), and vocabulary-based scoring over 4096--8192 discrete candidates (GTRS-Dense/AUG), each defining a qualitatively different mapping from pixel perturbations to trajectory outputs. The Jacobians $\partial\hat{\tau}/\partial\delta$ are consequently structurally dissimilar across families, such that a perturbation optimized through DiffusionDrive's 2-step DDIM chain has no reason to remain effective when processed by GTRS-DP's 100-step DDPM chain. These results confirm that transferability in generative AD is gated by shared backbone and preprocessing structure, and that the architectural heterogeneity of current generative planning models constitutes a meaningful, if unintentional, barrier to black-box adversarial attacks.

\subsection{Directions Toward More Transferable Attacks.}
We adopt a white-box threat model as our primary setting because it provides the gradient access needed to isolate and attribute each failure mode to a specific stage of the chain in Eq.~\eqref{eq:chain} (main paper); the transferability analysis is therefore complementary, serving as a diagnostic that characterizes the boundary conditions of \textsc{Derail}'s effectiveness beyond the white-box setting. This diagnostic reveals that cross-family transfer is constrained by the architectural heterogeneity of current generative planning models, a property of the model class itself that, as we show, can be partially addressed through ensemble optimization and surrogate model selection. Partial intra-family transfer within GTRS variants confirms that shared backbone structure and preprocessing pipelines are the primary enablers of cross-model generalization, providing a clear roadmap for improving transferability beyond the white-box setting. Specifically, ensemble-based optimization, jointly maximizing \textsc{Derail}'s objective across multiple surrogate models simultaneously, would encourage perturbations to exploit shared vulnerability directions rather than model-specific gradient pathways. Input diversity techniques such as random resizing and padding during optimization~\cite{xie2019improvingtransferabilityadversarialexamples} could further mitigate the resolution sensitivity between DiffusionDrive ($256\times1024$) and GTRS ($512\times2048$), encouraging perturbations that remain effective across 
different downsampling scales. Attacking at the BEV feature level rather than pixel level, if partial model access is available, would bypass backbone-specific frequency sensitivities entirely, as the BEV bottleneck constitutes a shared representation space across all evaluated architectures. Finally, the asymmetric transfer from GTRS-DP to GTRS-Dense suggests that surrogate models with longer iterative optimization horizons (e.g., full DDPM schedules) produce broader multi-scale feature corruption and may serve as more effective general-purpose surrogates for black-box attack crafting. Together, these directions suggest that transferable adversarial attacks against generative AD systems are achievable beyond the white-box setting, but require deliberate optimization strategies that account for the architectural diversity of this emerging model class.

%% file: tables/attacked_camera.tex
\begin{table}[t]
\centering
\caption{Attacking different camera views (Left, Front, Right).}
\vspace{-0.3cm}
\label{tab:attacked_camera}

\definecolor{oursaccent}{HTML}{8B2635}
\definecolor{oursrow}{HTML}{FDECEA}
\definecolor{modelband}{HTML}{FFFFFF}
\definecolor{modelbandtext}{HTML}{000000}

\setlength{\tabcolsep}{2.4pt}
\renewcommand{\arraystretch}{1.02}

\newcommand{\modelband}[2]{%
  \multicolumn{9}{c}{\cellcolor{modelband}%
    \color{modelbandtext}\textit{Attacked Model:}~\textbf{#1}~\cite{#2}} \\
  \arrayrulecolor{modelbandtext}\midrule\arrayrulecolor{black}%
}

\resizebox{\columnwidth}{!}{%
\footnotesize
\begin{tabular}{@{}l*{8}{c}@{}}
\toprule
\textbf{Target Camera}
 & \textbf{PDM}$\downarrow$
 & \textbf{Drop}$\uparrow$
 & \textbf{ASR}$\uparrow$
 & \textbf{CR}$\uparrow$
 & \textbf{OR}$\uparrow$
 & \textbf{WW}$\uparrow$
 & \textbf{TTC}$\uparrow$
 & \textbf{CM}$\uparrow$ \\
\midrule

\modelband{DiffusionDrive}{liao2025diffusiondrive}
Clean              & 0.93 & 0.00  & 0.00   & 0.87  & 1.63  & 4.88  & 4.90  & 0.00 \\
\texttt{Cam\_Left}  & 0.45 & 51.20 & 100.00 & 41.17 & 17.89 & 15.80 & 54.56 & 4.12 \\
\texttt{Cam\_Right} & 0.46 & 50.27 & 100.00 & 41.42 & 16.85 & 15.91 & 53.73 & 4.21 \\
\rowcolor{oursrow}
\textbf{\texttt{Cam\_Front}}
 & \textbf{\textcolor{oursaccent}{0.39}}
 & \textbf{\textcolor{oursaccent}{57.84}}
 & \textbf{\textcolor{oursaccent}{100.00}}
 & \textbf{\textcolor{oursaccent}{44.29}}
 & \textbf{\textcolor{oursaccent}{26.38}}
 & \textbf{\textcolor{oursaccent}{17.78}}
 & \textbf{\textcolor{oursaccent}{57.92}}
 & \textbf{\textcolor{oursaccent}{3.73}} \\
\cmidrule(l){1-9}

\modelband{GTRS-DP}{li2025generalized}
Clean              & 0.76 & 0.00  & 0.00   & 14.72 & 2.80  & 0.00  & 15.55 & 1.10 \\
\texttt{Cam\_Left}  & 0.62 & 19.21 & 100.00 & 25.44 & 9.27  & 2.34  & 28.18 & 1.31 \\
\texttt{Cam\_Right} & 0.71 & 7.27  & 80.00  & 18.87 & 4.33  & 0.83  & 21.00 & 1.33 \\
\rowcolor{oursrow}
\textbf{\texttt{Cam\_Front}}
 & \textbf{\textcolor{oursaccent}{0.20}}
 & \textbf{\textcolor{oursaccent}{73.77}}
 & \textbf{\textcolor{oursaccent}{100.00}}
 & \textbf{\textcolor{oursaccent}{54.05}}
 & \textbf{\textcolor{oursaccent}{51.71}}
 & \textbf{\textcolor{oursaccent}{33.13}}
 & \textbf{\textcolor{oursaccent}{61.19}}
 & \textbf{\textcolor{oursaccent}{18.21}} \\
\cmidrule(l){1-9}

\modelband{GTRS-AUG}{li2025generalized}
Clean              & 0.80 & 0.00  & 0.00   & 13.78 & 0.00  & 0.00  & 14.39 & 1.10 \\
\texttt{Cam\_Left}  & 0.71 & 11.60 & 90.00  & 22.66 & 0.90  & 0.08  & 24.50 & 0.97 \\
\texttt{Cam\_Right} & 0.70 & 12.06 & 90.00  & 23.34 & 0.58  & 0.00  & 24.63 & 1.12 \\
\rowcolor{oursrow}
\textbf{\texttt{Cam\_Front}}
 & \textbf{\textcolor{oursaccent}{0.28}}
 & \textbf{\textcolor{oursaccent}{65.13}}
 & \textbf{\textcolor{oursaccent}{100.00}}
 & \textbf{\textcolor{oursaccent}{44.64}}
 & \textbf{\textcolor{oursaccent}{42.63}}
 & \textbf{\textcolor{oursaccent}{25.92}}
 & \textbf{\textcolor{oursaccent}{50.54}}
 & \textbf{\textcolor{oursaccent}{9.24}} \\
\cmidrule(l){1-9}

\modelband{GTRS-Dense}{li2025generalized}
Clean              & 0.82 & 0.00  & 0.00   & 11.12 & 0.49  & 0.00  & 11.87 & 1.23 \\
\texttt{Cam\_Left}  & 0.67 & 17.69 & 100.00 & 18.57 & 4.04  & 0.18  & 20.43 & 15.19 \\
\texttt{Cam\_Right} & 0.68 & 17.13 & 90.00  & 17.18 & 5.04  & 0.57  & 18.45 & 14.27 \\
\rowcolor{oursrow}
\textbf{\texttt{Cam\_Front}}
 & \textbf{\textcolor{oursaccent}{0.44}}
 & \textbf{\textcolor{oursaccent}{46.62}}
 & \textbf{\textcolor{oursaccent}{100.00}}
 & \textbf{\textcolor{oursaccent}{37.27}}
 & \textbf{\textcolor{oursaccent}{10.46}}
 & \textbf{\textcolor{oursaccent}{3.24}}
 & \textbf{\textcolor{oursaccent}{38.74}}
 & \textbf{\textcolor{oursaccent}{40.33}} \\

\bottomrule
\end{tabular}}
\vspace{-0.4cm}
\end{table}

%% file: tables/universal_perturbation.tex
\begin{table}[t]
\centering
\caption{Online per-frame vs.\ offline universal perturbation.}
\vspace{-0.2cm}
\label{tab:universal_perturbation}

\definecolor{oursaccent}{HTML}{8B2635}
\definecolor{oursrow}{HTML}{FDECEA}
\definecolor{modelband}{HTML}{FFFFFF}
\definecolor{modelbandtext}{HTML}{000000}

\setlength{\tabcolsep}{2.4pt}
\renewcommand{\arraystretch}{1.02}

\newcommand{\modelband}[2]{%
  \multicolumn{9}{c}{\cellcolor{modelband}%
    \color{modelbandtext}\textit{Attacked Model:}~\textbf{#1}~\cite{#2}} \\
  \arrayrulecolor{modelbandtext}\midrule\arrayrulecolor{black}%
}

\resizebox{\columnwidth}{!}{%
\footnotesize
\begin{tabular}{@{}l*{8}{c}@{}}
\toprule
\textbf{Attack Strategy}
 & \textbf{PDM}$\downarrow$
 & \textbf{Drop}$\uparrow$
 & \textbf{ASR}$\uparrow$
 & \textbf{CR}$\uparrow$
 & \textbf{OR}$\uparrow$
 & \textbf{WW}$\uparrow$
 & \textbf{TTC}$\uparrow$
 & \textbf{CM}$\uparrow$ \\
\midrule

\modelband{DiffusionDrive}{liao2025diffusiondrive}
Clean              & 0.93 & 0.00  & 0.00   & 0.87  & 1.63  & 4.88  & 4.90  & 0.00 \\
Offline Universal  & 0.54 & 41.88 & 100.00 & 34.74 & 7.96  & 10.63 & 46.36 & 0.23 \\
\rowcolor{oursrow}
\textbf{Online Per-frame}
 & \textbf{\textcolor{oursaccent}{0.39}}
 & \textbf{\textcolor{oursaccent}{57.84}}
 & \textbf{\textcolor{oursaccent}{100.00}}
 & \textbf{\textcolor{oursaccent}{44.29}}
 & \textbf{\textcolor{oursaccent}{26.38}}
 & \textbf{\textcolor{oursaccent}{17.78}}
 & \textbf{\textcolor{oursaccent}{57.92}}
 & \textbf{\textcolor{oursaccent}{3.73}} \\
\cmidrule(l){1-9}

\modelband{GTRS-DP}{li2025generalized}
Clean              & 0.76 & 0.00  & 0.00   & 14.72 & 2.80  & 0.00  & 15.55 & 1.10  \\
Offline Universal  & 0.19 & 75.60 & 100.00 & 44.71 & 68.87 & 50.73 & 53.88 & 18.75 \\
\rowcolor{oursrow}
\textbf{Online Per-frame}
 & \textbf{\textcolor{oursaccent}{0.20}}
 & \textbf{\textcolor{oursaccent}{73.77}}
 & \textbf{\textcolor{oursaccent}{100.00}}
 & \textbf{\textcolor{oursaccent}{54.05}}
 & \textbf{\textcolor{oursaccent}{51.71}}
 & \textbf{\textcolor{oursaccent}{33.13}}
 & \textbf{\textcolor{oursaccent}{61.19}}
 & \textbf{\textcolor{oursaccent}{18.21}} \\
\cmidrule(l){1-9}

\modelband{GTRS-AUG}{li2025generalized}
Clean              & 0.80 & 0.00  & 0.00   & 13.78 & 0.00  & 0.00  & 14.39 & 1.10 \\
Offline Universal  & 0.39 & 50.84 & 90.00  & 35.99 & 34.00 & 18.63 & 41.07 & 4.62 \\
\rowcolor{oursrow}
\textbf{Online Per-frame}
 & \textbf{\textcolor{oursaccent}{0.28}}
 & \textbf{\textcolor{oursaccent}{65.13}}
 & \textbf{\textcolor{oursaccent}{100.00}}
 & \textbf{\textcolor{oursaccent}{44.64}}
 & \textbf{\textcolor{oursaccent}{42.63}}
 & \textbf{\textcolor{oursaccent}{25.92}}
 & \textbf{\textcolor{oursaccent}{50.54}}
 & \textbf{\textcolor{oursaccent}{9.24}} \\
\cmidrule(l){1-9}

\modelband{GTRS-Dense}{li2025generalized}
Clean              & 0.82 & 0.00  & 0.00   & 11.12 & 0.49  & 0.00 & 11.87 & 1.23  \\
Offline Universal  & 0.59 & 28.05 & 90.00  & 15.24 & 4.82  & 0.95 & 15.85 & 37.89 \\
\rowcolor{oursrow}
\textbf{Online Per-frame}
 & \textbf{\textcolor{oursaccent}{0.44}}
 & \textbf{\textcolor{oursaccent}{46.62}}
 & \textbf{\textcolor{oursaccent}{100.00}}
 & \textbf{\textcolor{oursaccent}{37.27}}
 & \textbf{\textcolor{oursaccent}{10.46}}
 & \textbf{\textcolor{oursaccent}{3.24}}
 & \textbf{\textcolor{oursaccent}{38.74}}
 & \textbf{\textcolor{oursaccent}{40.33}} \\

\bottomrule
\end{tabular}}
\vspace{-0.4cm}
\end{table}

%% file: tables/defenses.tex
\begin{table}[t]
\centering
\caption{Impact of different input-level defenses (CJ, SS, and GN) on the \textsc{Derail} attack across different planning models.}
\vspace{-0.2cm}
\label{tab:defense_results}

\definecolor{oursaccent}{HTML}{8B2635}
\definecolor{oursrow}{HTML}{FDECEA}
\definecolor{modelband}{HTML}{FFFFFF}
\definecolor{modelbandtext}{HTML}{000000}

\setlength{\tabcolsep}{2.4pt}
\renewcommand{\arraystretch}{1.02}

\newcommand{\modelband}[2]{%
  \multicolumn{9}{c}{\cellcolor{modelband}%
    \color{modelbandtext}\textit{Attacked Model:}~\textbf{#1}~\cite{#2}} \\
  \arrayrulecolor{modelbandtext}\midrule\arrayrulecolor{black}%
}

\resizebox{\columnwidth}{!}{%
\footnotesize
\begin{tabular}{@{}l*{8}{c}@{}}
\toprule
\textbf{Attack / Defense}
 & \textbf{PDM}$\downarrow$
 & \textbf{Drop}$\uparrow$
 & \textbf{ASR}$\uparrow$
 & \textbf{CR}$\uparrow$
 & \textbf{OR}$\uparrow$
 & \textbf{WW}$\uparrow$
 & \textbf{TTC}$\uparrow$
 & \textbf{CM}$\uparrow$ \\
\midrule

\modelband{DiffusionDrive}{liao2025diffusiondrive}
Clean                  & 0.93 & 0.00  & 0.00   & 0.87  & 1.63  & 4.88  & 4.90  & 0.00 \\
\texttt{Derail\_CJ}    & 0.50 & 45.89 & 100.00 & 36.69 & 14.54 & 10.27 & 49.48 & 2.71 \\
\texttt{Derail\_SS}    & 0.62 & 33.01 & 100.00 & 27.15 & 7.93  & 6.69  & 38.63 & 1.33 \\
\texttt{Derail\_GN}    & 0.55 & 41.23 & 100.00 & 34.46 & 10.51 & 9.18  & 46.11 & 2.41 \\
\rowcolor{oursrow}
\textbf{\texttt{Derail\_None}}
 & \textbf{\textcolor{oursaccent}{0.39}}
 & \textbf{\textcolor{oursaccent}{57.84}}
 & \textbf{\textcolor{oursaccent}{100.00}}
 & \textbf{\textcolor{oursaccent}{44.29}}
 & \textbf{\textcolor{oursaccent}{26.38}}
 & \textbf{\textcolor{oursaccent}{17.78}}
 & \textbf{\textcolor{oursaccent}{57.92}}
 & \textbf{\textcolor{oursaccent}{3.73}} \\
\cmidrule(l){1-9}

\modelband{GTRS-DP}{li2025generalized}
Clean                  & 0.76 & 0.00  & 0.00   & 14.72 & 2.80  & 0.00  & 15.55 & 1.10  \\
\texttt{Derail\_CJ}    & 0.22 & 71.23 & 100.00 & 53.51 & 48.96 & 30.19 & 61.73 & 16.92 \\
\texttt{Derail\_SS}    & 0.21 & 72.67 & 100.00 & 54.70 & 50.21 & 31.92 & 61.52 & 18.89 \\
\texttt{Derail\_GN}    & 0.26 & 65.41 & 100.00 & 53.31 & 39.80 & 22.90 & 60.42 & 13.37 \\
\rowcolor{oursrow}
\textbf{\texttt{Derail\_None}}
 & \textbf{\textcolor{oursaccent}{0.20}}
 & \textbf{\textcolor{oursaccent}{73.77}}
 & \textbf{\textcolor{oursaccent}{100.00}}
 & \textbf{\textcolor{oursaccent}{54.05}}
 & \textbf{\textcolor{oursaccent}{51.71}}
 & \textbf{\textcolor{oursaccent}{33.13}}
 & \textbf{\textcolor{oursaccent}{61.19}}
 & \textbf{\textcolor{oursaccent}{18.21}} \\
\cmidrule(l){1-9}

\modelband{GTRS-AUG}{li2025generalized}
Clean                  & 0.80 & 0.00  & 0.00   & 13.78 & 0.00  & 0.00  & 14.39 & 1.10 \\
\texttt{Derail\_CJ}    & 0.32 & 60.28 & 100.00 & 41.46 & 37.33 & 20.24 & 47.91 & 8.43 \\
\texttt{Derail\_SS}    & 0.30 & 62.84 & 100.00 & 42.42 & 37.93 & 22.08 & 49.70 & 8.10 \\
\texttt{Derail\_GN}    & 0.31 & 61.40 & 100.00 & 42.94 & 35.78 & 20.85 & 49.59 & 6.41 \\
\rowcolor{oursrow}
\textbf{\texttt{Derail\_None}}
 & \textbf{\textcolor{oursaccent}{0.28}}
 & \textbf{\textcolor{oursaccent}{65.13}}
 & \textbf{\textcolor{oursaccent}{100.00}}
 & \textbf{\textcolor{oursaccent}{44.64}}
 & \textbf{\textcolor{oursaccent}{42.63}}
 & \textbf{\textcolor{oursaccent}{25.92}}
 & \textbf{\textcolor{oursaccent}{50.54}}
 & \textbf{\textcolor{oursaccent}{9.24}} \\
\cmidrule(l){1-9}

\modelband{GTRS-Dense}{li2025generalized}
Clean                  & 0.82 & 0.00  & 0.00   & 11.12 & 0.49  & 0.00 & 11.87 & 1.23  \\
\texttt{Derail\_CJ}    & 0.43 & 47.47 & 100.00 & 37.64 & 10.23 & 2.83 & 39.98 & 43.98 \\
\texttt{Derail\_SS}    & 0.44 & 45.97 & 100.00 & 38.77 & 10.19 & 3.34 & 39.81 & 44.17 \\
\texttt{Derail\_GN}    & 0.40 & 50.92 & 100.00 & 41.16 & 11.58 & 3.31 & 43.00 & 43.31 \\
\rowcolor{oursrow}
\textbf{\texttt{Derail\_None}}
 & \textbf{\textcolor{oursaccent}{0.44}}
 & \textbf{\textcolor{oursaccent}{46.62}}
 & \textbf{\textcolor{oursaccent}{100.00}}
 & \textbf{\textcolor{oursaccent}{37.27}}
 & \textbf{\textcolor{oursaccent}{10.46}}
 & \textbf{\textcolor{oursaccent}{3.24}}
 & \textbf{\textcolor{oursaccent}{38.74}}
 & \textbf{\textcolor{oursaccent}{40.33}} \\

\bottomrule
\end{tabular}}
\vspace{-0.4cm}
\end{table}

%% file: tables/transferrability.tex
\begin{table}[t]
\centering
\caption{Black-box transferability analysis of \textsc{Derail}.}
\vspace{-0.2cm}
\label{tab:transferability_matrix}

\definecolor{oursaccent}{HTML}{8B2635}
\definecolor{oursrow}{HTML}{FDECEA}

\setlength{\tabcolsep}{2.4pt}
\renewcommand{\arraystretch}{1.05}

\resizebox{\columnwidth}{!}{%
\footnotesize
\begin{tabular}{@{}l@{\hspace{6pt}}*{4}{c}@{\hspace{8pt}}*{4}{c}@{\hspace{8pt}}*{4}{c}@{}}
\toprule
\multirow{2}{*}{\textbf{White-box Surrogate}}
 & \multicolumn{12}{c}{\textbf{Black-box Target Models} $\rightarrow$} \\
\cmidrule(lr){2-13}
 & \multicolumn{4}{c}{\textbf{GTRS-DP}~\cite{li2025generalized}}
 & \multicolumn{4}{c}{\textbf{GTRS-AUG}~\cite{li2025generalized}}
 & \multicolumn{4}{c}{\textbf{GTRS-Dense}~\cite{li2025generalized}} \\
\cmidrule(lr){2-5}\cmidrule(lr){6-9}\cmidrule(lr){10-13}
 & PDM$\downarrow$ & Drop$\uparrow$ & ASR$\uparrow$ & CR$\uparrow$
 & PDM$\downarrow$ & Drop$\uparrow$ & ASR$\uparrow$ & CR$\uparrow$
 & PDM$\downarrow$ & Drop$\uparrow$ & ASR$\uparrow$ & CR$\uparrow$ \\
\midrule

\textit{Clean (no attack)}
 & 0.87 & 0.00 & 0.00 & 5.06
 & 0.84 & 0.00 & 0.00 & 11.39
 & 0.86 & 0.00 & 0.00 & 6.33 \\
\midrule

\textbf{GTRS-DP}
 & \cellcolor{oursrow}\textbf{\textcolor{oursaccent}{0.09}}
 & \cellcolor{oursrow}\textbf{\textcolor{oursaccent}{0.78}}
 & \cellcolor{oursrow}\textbf{\textcolor{oursaccent}{100}}
 & \cellcolor{oursrow}\textbf{\textcolor{oursaccent}{66.46}}
 & 0.80 & 0.03 & 100 & 15.19
 & 0.78 & 0.09 & 100 & 10.13 \\

\textbf{GTRS-AUG}
 & 0.84 & 0.03 & 100 & 7.59
 & \cellcolor{oursrow}\textbf{\textcolor{oursaccent}{0.49}}
 & \cellcolor{oursrow}\textbf{\textcolor{oursaccent}{0.35}}
 & \cellcolor{oursrow}\textbf{\textcolor{oursaccent}{100}}
 & \cellcolor{oursrow}\textbf{\textcolor{oursaccent}{30.38}}
 & 0.84 & 0.02 & 100 & 7.59 \\

\textbf{GTRS-Dense}
 & 0.81 & 0.07 & 100 & 11.39
 & 0.81 & 0.03 & 100 & 15.19
 & \cellcolor{oursrow}\textbf{\textcolor{oursaccent}{0.57}}
 & \cellcolor{oursrow}\textbf{\textcolor{oursaccent}{0.29}}
 & \cellcolor{oursrow}\textbf{\textcolor{oursaccent}{100}}
 & \cellcolor{oursrow}\textbf{\textcolor{oursaccent}{20.25}} \\

\bottomrule
\end{tabular}}
\vspace{-0.5cm}
\end{table}

%% file: main.bib
@String(AAAI  = {AAAI})

@article{sohn2015learning,
  title={Learning structured output representation using deep conditional generative models},
  author={Sohn, Kihyuk and Lee, Honglak and Yan, Xinchen},
  journal={Advances in neural information processing systems},
  volume={28},
  year={2015}
}

@inproceedings{ajayconditional,
  title={Is Conditional Generative Modeling all you need for Decision Making?},
  author={Ajay, Anurag and Du, Yilun and Gupta, Abhi and Tenenbaum, Joshua B and Jaakkola, Tommi S and Agrawal, Pulkit},
  booktitle={The Eleventh International Conference on Learning Representations}
}

@article{ivanovic2020multimodal,
  title={Multimodal deep generative models for trajectory prediction: A conditional variational autoencoder approach},
  author={Ivanovic, Boris and Leung, Karen and Schmerling, Edward and Pavone, Marco},
  journal={IEEE Robotics and Automation Letters},
  volume={6},
  number={2},
  pages={295--302},
  year={2020},
  publisher={IEEE}
}

@article{he2023diffusion,
  title={Diffusion model is an effective planner and data synthesizer for multi-task reinforcement learning},
  author={He, Haoran and Bai, Chenjia and Xu, Kang and Yang, Zhuoran and Zhang, Weinan and Wang, Dong and Zhao, Bin and Li, Xuelong},
  journal={Advances in neural information processing systems},
  volume={36},
  pages={64896--64917},
  year={2023}
}

@inproceedings{li2019conditional,
  title={Conditional generative neural system for probabilistic trajectory prediction},
  author={Li, Jiachen and Ma, Hengbo and Tomizuka, Masayoshi},
  booktitle={2019 IEEE/RSJ International Conference on Intelligent Robots and Systems (IROS)},
  pages={6150--6156},
  year={2019},
  organization={IEEE}
}

@inproceedings{wangrobogen,
  title={RoboGen: Towards Unleashing Infinite Data for Automated Robot Learning via Generative Simulation},
  author={Wang, Yufei and Xian, Zhou and Chen, Feng and Wang, Tsun-Hsuan and Wang, Yian and Fragkiadaki, Katerina and Erickson, Zackory and Held, David and Gan, Chuang},
  booktitle={Forty-first International Conference on Machine Learning}
}

@article{zhang2025generative,
  title={Generative artificial intelligence in robotic manipulation: A survey},
  author={Zhang, Kun and Yun, Peng and Cen, Jun and Cai, Junhao and Zhu, Didi and Yuan, Hangjie and Zhao, Chao and Feng, Tao and Wang, Michael Yu and Chen, Qifeng and others},
  journal={arXiv preprint arXiv:2503.03464},
  year={2025}
}

@article{gomez2020real,
  title={Real time trajectory prediction using deep conditional generative models},
  author={Gomez-Gonzalez, Sebastian and Prokudin, Sergey and Sch{\"o}lkopf, Bernhard and Peters, Jan},
  journal={IEEE Robotics and Automation Letters},
  volume={5},
  number={2},
  pages={970--976},
  year={2020},
  publisher={IEEE}
}

@inproceedings{arnelid2019recurrent,
  title={Recurrent conditional generative adversarial networks for autonomous driving sensor modelling},
  author={Arnelid, Henrik and Zec, Edvin Listo and Mohammadiha, Nasser},
  booktitle={2019 IEEE Intelligent transportation systems conference (ITSC)},
  pages={1613--1618},
  year={2019},
  organization={IEEE}
}

@inproceedings{zheng2024genad,
  title={Genad: Generative end-to-end autonomous driving},
  author={Zheng, Wenzhao and Song, Ruiqi and Guo, Xianda and Zhang, Chenming and Chen, Long},
  booktitle={European Conference on Computer Vision},
  pages={87--104},
  year={2024},
  organization={Springer}
}

@article{wang2025generative,
  title={Generative ai for autonomous driving: Frontiers and opportunities},
  author={Wang, Yuping and Xing, Shuo and Can, Cui and Li, Renjie and Hua, Hongyuan and Tian, Kexin and Mo, Zhaobin and Gao, Xiangbo and Wu, Keshu and Zhou, Sulong and others},
  journal={arXiv preprint arXiv:2505.08854},
  year={2025}
}

@article{janner2022planning,
  title={Planning with diffusion for flexible behavior synthesis},
  author={Janner, Michael and Du, Yilun and Tenenbaum, Joshua B and Levine, Sergey},
  journal={arXiv preprint arXiv:2205.09991},
  year={2022}
}

@article{ho2020denoising,
  title={Denoising diffusion probabilistic models},
  author={Ho, Jonathan and Jain, Ajay and Abbeel, Pieter},
  journal={Advances in neural information processing systems},
  volume={33},
  pages={6840--6851},
  year={2020}
}

@article{chi2025diffusion,
  title={Diffusion policy: Visuomotor policy learning via action diffusion},
  author={Chi, Cheng and Xu, Zhenjia and Feng, Siyuan and Cousineau, Eric and Du, Yilun and Burchfiel, Benjamin and Tedrake, Russ and Song, Shuran},
  journal={The International Journal of Robotics Research},
  volume={44},
  number={10-11},
  pages={1684--1704},
  year={2025},
  publisher={Sage Publications Sage UK: London, England}
}

@inproceedings{li2024crossway,
  title={Crossway diffusion: Improving diffusion-based visuomotor policy via self-supervised learning},
  author={Li, Xiang and Belagali, Varun and Shang, Jinghuan and Ryoo, Michael S},
  booktitle={2024 IEEE International Conference on Robotics and Automation (ICRA)},
  pages={16841--16849},
  year={2024},
  organization={IEEE}
}

@article{liu2024ddm,
  title={DDM-lag: A diffusion-based decision-making model for autonomous vehicles with lagrangian safety enhancement},
  author={Liu, Jiaqi and Hang, Peng and Zhao, Xiaocong and Wang, Jianqiang and Sun, Jian},
  journal={IEEE Transactions on Artificial Intelligence},
  volume={6},
  number={3},
  pages={780--791},
  year={2024},
  publisher={IEEE}
}

@inproceedings{luo2024pot,
  title={Pot potential based diffusion motion planning ential based diffusion motion planning},
  author={Luo, Yunhao and Sun, Chen and Tenenbaum, Joshua B and Du, Yilun},
  booktitle={Proceedings of the 41st International Conference on Machine Learning},
  pages={33486--33510},
  year={2024}
}

@inproceedings{mishra2023generative,
  title={Generative skill chaining: Long-horizon skill planning with diffusion models},
  author={Mishra, Utkarsh Aashu and Xue, Shangjie and Chen, Yongxin and Xu, Danfei},
  booktitle={Conference on Robot Learning},
  pages={2905--2925},
  year={2023},
  organization={PMLR}
}

@inproceedings{zhengdiffusion,
  title={Diffusion-Based Planning for Autonomous Driving with Flexible Guidance},
  author={Zheng, Yinan and Liang, Ruiming and ZHENG, Kexin and Zheng, Jinliang and Mao, Liyuan and Li, Jianxiong and Gu, Weihao and Ai, Rui and Li, Shengbo Eben and Zhan, Xianyuan and others},
  booktitle={The Thirteenth International Conference on Learning Representations}
}

@inproceedings{liao2025diffusiondrive,
  title={Diffusiondrive: Truncated diffusion model for end-to-end autonomous driving},
  author={Liao, Bencheng and Chen, Shaoyu and Yin, Haoran and Jiang, Bo and Wang, Cheng and Yan, Sixu and Zhang, Xinbang and Li, Xiangyu and Zhang, Ying and Zhang, Qian and others},
  booktitle={Proceedings of the Computer Vision and Pattern Recognition Conference},
  pages={12037--12047},
  year={2025}
}

@article{wang2025diffad,
  title={Diffad: A unified diffusion modeling approach for autonomous driving},
  author={Wang, Tao and Zhang, Cong and Qu, Xingguang and Li, Kun and Liu, Weiwei and Huang, Chang},
  journal={arXiv preprint arXiv:2503.12170},
  year={2025}
}

@article{li2025generalized,
  title={Generalized trajectory scoring for end-to-end multimodal planning},
  author={Li, Zhenxin and Yao, Wenhao and Wang, Zi and Sun, Xinglong and Chen, Joshua and Chang, Nadine and Shen, Maying and Wu, Zuxuan and Lan, Shiyi and Alvarez, Jose M},
  journal={arXiv preprint arXiv:2506.06664},
  year={2025}
}

@article{zheng2026unleashing,
  title={Unleashing the Potential of Diffusion Models for End-to-End Autonomous Driving},
  author={Zheng, Yinan and Tan, Tianyi and Huang, Bin and Liu, Enguang and Liang, Ruiming and Zhang, Jianlin and Cui, Jianwei and Chen, Guang and Ma, Kun and Ye, Hangjun and others},
  journal={arXiv preprint arXiv:2602.22801},
  year={2026}
}

@article{yao2026reflexdiffusion,
  title={ReflexDiffusion: Reflection-Enhanced Trajectory Planning for High-lateral-acceleration Scenarios in Autonomous Driving},
  author={Yao, Xuemei and Yang, Xiao and Sun, Jianbin and Xie, Liuwei and Shao, Xuebin and Fang, Xiyu and Su, Hang and Yang, Kewei},
  journal={arXiv preprint arXiv:2601.09377},
  year={2026}
}

@article{reddy2026rapid,
  title={RAPiD: Real-time Deterministic Trajectory Planning via Diffusion Behavior Priors for Safe and Efficient Autonomous Driving},
  author={Reddy, Ruturaj and Barua, Hrishav Bakul and Loo, Junn Yong and Nguyen, Thanh Thi and Krishnasamy, Ganesh},
  journal={arXiv preprint arXiv:2602.07339},
  year={2026}
}

@inproceedings{jiang2023motiondiffuser,
  title={Motiondiffuser: Controllable multi-agent motion prediction using diffusion},
  author={Jiang, Chiyu and Cornman, Andre and Park, Cheolho and Sapp, Benjamin and Zhou, Yin and Anguelov, Dragomir and others},
  booktitle={Proceedings of the IEEE/CVF conference on computer vision and pattern recognition},
  pages={9644--9653},
  year={2023}
}

@article{ren2025cosmos,
  title={Cosmos-drive-dreams: Scalable synthetic driving data generation with world foundation models},
  author={Ren, Xuanchi and Lu, Yifan and Cao, Tianshi and Gao, Ruiyuan and Huang, Shengyu and Sabour, Amirmojtaba and Shen, Tianchang and Pfaff, Tobias and Wu, Jay Zhangjie and Chen, Runjian and others},
  journal={arXiv preprint arXiv:2506.09042},
  year={2025}
}

@article{chen2024diffusion,
  title={Diffusion policy attacker: Crafting adversarial attacks for diffusion-based policies},
  author={Chen, Yipu and Xue, Haotian and Chen, Yongxin},
  journal={Advances in Neural Information Processing Systems},
  volume={37},
  pages={119614--119637},
  year={2024}
}

@article{gleave2019adversarial,
  title={Adversarial policies: Attacking deep reinforcement learning},
  author={Gleave, Adam and Dennis, Michael and Wild, Cody and Kant, Neel and Levine, Sergey and Russell, Stuart},
  journal={arXiv preprint arXiv:1905.10615},
  year={2019}
}

@inproceedings{madry2018towards,
  title={Towards Deep Learning Models Resistant to Adversarial Attacks},
  author={Madry, Aleksander and Makelov, Aleksandar and Schmidt, Ludwig and Tsipras, Dimitris and Vladu, Adrian},
  booktitle={International Conference on Learning Representations},
  year={2018}
}

@article{mo2022attacking,
  title={Attacking deep reinforcement learning with decoupled adversarial policy},
  author={Mo, Kanghua and Tang, Weixuan and Li, Jin and Yuan, Xu},
  journal={IEEE Transactions on Dependable and Secure Computing},
  volume={20},
  number={1},
  pages={758--768},
  year={2022},
  publisher={IEEE}
}

@inproceedings{pattanaik2018robust,
  title={Robust Deep Reinforcement Learning with Adversarial Attacks},
  author={Pattanaik, Anay and Tang, Zhenyi and Liu, Shuijing and Bommannan, Gautham and Chowdhary, Girish},
  booktitle={Proceedings of the 17th International Conference on Autonomous Agents and MultiAgent Systems},
  pages={2040--2042},
  year={2018}
}

@inproceedings{pan2018agile,
  title={Agile Autonomous Driving using End-to-End Deep Imitation Learning},
  author={Pan, Yunpeng and Cheng, Ching-An and Saigol, Kamil and Lee, Keuntak and Yan, Xinyan and Theodorou, Evangelos and Boots, Byron},
  booktitle={Robotics: science and systems},
  year={2018}
}

@article{schaal1999imitation,
  title={Is imitation learning the route to humanoid robots?},
  author={Schaal, Stefan},
  journal={Trends in cognitive sciences},
  volume={3},
  number={6},
  pages={233--242},
  year={1999},
  publisher={Elsevier}
}

@article{osa2018algorithmic,
  title={An algorithmic perspective on imitation learning},
  author={Osa, Takayuki and Pajarinen, Joni and Neumann, Gerhard and Bagnell, J Andrew and Abbeel, Pieter and Peters, Jan},
  journal={Foundations and Trends{\textregistered} in Robotics},
  volume={7},
  number={1-2},
  pages={1--179},
  year={2018},
  publisher={Emerald Publishing Limited}
}

@inproceedings{songdenoising,
  title={Denoising Diffusion Implicit Models},
  author={Song, Jiaming and Meng, Chenlin and Ermon, Stefano},
  booktitle={International Conference on Learning Representations}
}

@inproceedings{songscore,
  title={Score-Based Generative Modeling through Stochastic Differential Equations},
  author={Song, Yang and Sohl-Dickstein, Jascha and Kingma, Diederik P and Kumar, Abhishek and Ermon, Stefano and Poole, Ben},
  booktitle={International Conference on Learning Representations}
}

@inproceedings{sun2020stealthy,
  title={Stealthy and efficient adversarial attacks against deep reinforcement learning},
  author={Sun, Jianwen and Zhang, Tianwei and Xie, Xiaofei and Ma, Lei and Zheng, Yan and Chen, Kangjie and Liu, Yang},
  booktitle={Proceedings of the AAAI conference on artificial intelligence},
  volume={34},
  number={04},
  pages={5883--5891},
  year={2020}
}

@article{szegedy2013intriguing,
  title={Intriguing properties of neural networks},
  author={Szegedy, Christian and Zaremba, Wojciech and Sutskever, Ilya and Bruna, Joan and Erhan, Dumitru and Goodfellow, Ian and Fergus, Rob},
  journal={arXiv preprint arXiv:1312.6199},
  year={2013}
}

@article{brown2017adversarial,
  title={Adversarial patch},
  author={Brown, Tom B and Man{\'e}, Dandelion and Roy, Aurko and Abadi, Mart{\'\i}n and Gilmer, Justin},
  journal={arXiv preprint arXiv:1712.09665},
  year={2017}
}

@article{liu2018dpatch,
  title={Dpatch: An adversarial patch attack on object detectors},
  author={Liu, Xin and Yang, Huanrui and Liu, Ziwei and Song, Linghao and Li, Hai and Chen, Yiran},
  journal={arXiv preprint arXiv:1806.02299},
  year={2018}
}

@inproceedings{hu2023planning,
  title={Planning-oriented autonomous driving},
  author={Hu, Yihan and Yang, Jiazhi and Chen, Li and Li, Keyu and Sima, Chonghao and Zhu, Xizhou and Chai, Siqi and Du, Senyao and Lin, Tianwei and Wang, Wenhai and others},
  booktitle={Proceedings of the IEEE/CVF conference on computer vision and pattern recognition},
  pages={17853--17862},
  year={2023}
}

@inproceedings{jiang2023vad,
  title={Vad: Vectorized scene representation for efficient autonomous driving},
  author={Jiang, Bo and Chen, Shaoyu and Xu, Qing and Liao, Bencheng and Chen, Jiajie and Zhou, Helong and Zhang, Qian and Liu, Wenyu and Huang, Chang and Wang, Xinggang},
  booktitle={Proceedings of the IEEE/CVF International Conference on Computer Vision},
  pages={8340--8350},
  year={2023}
}

@inproceedings{wang2024attack,
  title={Attack end-to-end autonomous driving through module-wise noise},
  author={Wang, Lu and Zhang, Tianyuan and Han, Yikai and Fang, Muyang and Jin, Ting and Kang, Jiaqi},
  booktitle={Proceedings of the IEEE/CVF Conference on Computer Vision and Pattern Recognition},
  pages={8349--8352},
  year={2024}
}

@inproceedings{wu2023adversarial,
  title={Adversarial driving: Attacking end-to-end autonomous driving},
  author={Wu, Han and Yunas, Syed and Rowlands, Sareh and Ruan, Wenjie and Wahlstr{\"o}m, Johan},
  booktitle={2023 IEEE intelligent vehicles symposium (IV)},
  pages={1--7},
  year={2023},
  organization={IEEE}
}

@article{zhang2024uniada,
  title={Uniada: Universal adaptive multiobjective adversarial attack for end-to-end autonomous driving systems},
  author={Zhang, Jingyu and Keung, Jacky Wai and Xiao, Yan and Liao, Yihan and Li, Yishu and Ma, Xiaoxue},
  journal={IEEE Transactions on Reliability},
  volume={73},
  number={4},
  pages={1892--1906},
  year={2024},
  publisher={IEEE}
}

@article{chahe2023dynamic,
  title={Dynamic adversarial attacks on autonomous driving systems},
  author={Chahe, Amirhosein and Wang, Chenan and Jeyapratap, Abhishek and Xu, Kaidi and Zhou, Lifeng},
  journal={arXiv preprint arXiv:2312.06701},
  year={2023}
}

@inproceedings{athalye2018synthesizing,
  title={Synthesizing robust adversarial examples},
  author={Athalye, Anish and Engstrom, Logan and Ilyas, Andrew and Kwok, Kevin},
  booktitle={International conference on machine learning},
  pages={284--293},
  year={2018},
  organization={PMLR}
}

@article{kingma2014adam,
  title={Adam: A method for stochastic optimization},
  author={Kingma, Diederik P and Ba, Jimmy},
  journal={arXiv preprint arXiv:1412.6980},
  year={2014}
}

@article{dauner2024navsim,
  title={Navsim: Data-driven non-reactive autonomous vehicle simulation and benchmarking},
  author={Dauner, Daniel and Hallgarten, Marcel and Li, Tianyu and Weng, Xinshuo and Huang, Zhiyu and Yang, Zetong and Li, Hongyang and Gilitschenski, Igor and Ivanovic, Boris and Pavone, Marco and others},
  journal={Advances in Neural Information Processing Systems},
  volume={37},
  pages={28706--28719},
  year={2024}
}

@article{caesar2021nuplan,
  title={nuplan: A closed-loop ml-based planning benchmark for autonomous vehicles},
  author={Caesar, Holger and Kabzan, Juraj and Tan, Kok Seang and Fong, Whye Kit and Wolff, Eric and Lang, Alex and Fletcher, Luke and Beijbom, Oscar and Omari, Sammy},
  journal={arXiv preprint arXiv:2106.11810},
  year={2021}
}

@inproceedings{dauner2023parting,
  title={Parting with misconceptions about learning-based vehicle motion planning},
  author={Dauner, Daniel and Hallgarten, Marcel and Geiger, Andreas and Chitta, Kashyap},
  booktitle={Conference on Robot Learning},
  pages={1268--1281},
  year={2023},
  organization={PMLR}
}

@article{virmaux2018lipschitz,
  title={Lipschitz regularity of deep neural networks: analysis and efficient estimation},
  author={Virmaux, Aladin and Scaman, Kevin},
  journal={Advances in neural information processing systems},
  volume={31},
  year={2018}
}

@inproceedings{yang2024vidar,
  title={Visual Point Cloud Forecasting enables Scalable Autonomous Driving},
  author={Yang, Zetong and Chen, Li and Sun, Yanan and Li, Hongyang},
  booktitle={Proceedings of the IEEE/CVF Conference on Computer Vision and Pattern Recognition},
  year={2024}
}

@misc{openscene2023,
  title={OpenScene: The Largest Up-to-Date 3D Occupancy Prediction Benchmark in Autonomous Driving},
  author={OpenScene Contributors},
  howpublished={\url{https://github.com/OpenDriveLab/OpenScene}},
  year={2023}
}

@article{sima2023_occnet,
  title={Scene as Occupancy}, 
  author={Chonghao Sima and Wenwen Tong and Tai Wang and Li Chen and Silei Wu and Hanming Deng  and Yi Gu and Lewei Lu and Ping Luo and Dahua Lin and Hongyang Li},
  year={2023},
  eprint={2306.02851},
  archivePrefix={arXiv},
  primaryClass={cs.CV}
}

@article{li2025hydra,
  title={Hydra-mdp++: Advancing end-to-end driving via expert-guided hydra-distillation},
  author={Li, Kailin and Li, Zhenxin and Lan, Shiyi and Xie, Yuan and Zhang, Zhizhong and Liu, Jiayi and Wu, Zuxuan and Yu, Zhiding and Alvarez, Jose M},
  journal={arXiv preprint arXiv:2503.12820},
  year={2025}
}

@misc{xie2019improvingtransferabilityadversarialexamples,
      title={Improving Transferability of Adversarial Examples with Input Diversity}, 
      author={Cihang Xie and Zhishuai Zhang and Yuyin Zhou and Song Bai and Jianyu Wang and Zhou Ren and Alan Yuille},
      year={2019},
      eprint={1803.06978},
      archivePrefix={arXiv},
      primaryClass={cs.CV},
      url={https://arxiv.org/abs/1803.06978}, 
}

@inproceedings{bouzidi2026out,
  title={Out of Sight, Out of Track: Adversarial Attacks on Propagation-based Multi-Object Trackers via Query State Manipulation},
  author={Bouzidi, Halima and Liu, Haoyu and Achamyeleh, Yonatan and Iddamsetty, Praneetsai and Al Faruque, Mohammad},
  booktitle={Proceedings of the IEEE/CVF Conference on Computer Vision and Pattern Recognition},
  pages={13326--13335},
  year={2026}
}

@article{bouzidi2025see,
  title={See No Evil: Adversarial Attacks Against Linguistic-Visual Association in Referring Multi-Object Tracking Systems},
  author={Bouzidi, Halima and Liu, Haoyu and Faruque, Mohammad Abdullah Al},
  journal={arXiv preprint arXiv:2509.02028},
  year={2025}
}
